\documentclass{article}
\usepackage{log_2026}						

\usepackage{booktabs}						
\usepackage{multirow}						
\usepackage{amsfonts}	
\usepackage{amsthm}
\usepackage{graphicx}						

\usepackage[numbers,compress,sort]{natbib}	

\title[Hyperparameter Transfer in Graph Neural Networks]{Hyperparameter Transfer in Graph Neural Networks}

\author[DeZoort et al.]{%
Gage DeZoort\and Boris Hanin\\
Department of Operations Research and Financial Engineering, Princeton University \\
Sherrerd Hall, Princeton, NJ 08540\\
\email{\{jdezoort,bhanin\}@princeton.edu}
}

\begin{document}


\newcommand{\cmark}{\ding{51}}%
\newcommand{\xmark}{\ding{55}}%
\newcommand{\Ls}[0]{\mathcal{L}}
\newcommand{\Xs}[0]{\mathcal{X}}
\newcommand{\R}[0]{\mathbb{R}}
\newcommand{\C}[0]{\mathbb{C}}
\newcommand{\Hs}[0]{\mathcal{H}}
\newcommand{\Yl}[0]{Y_{\ell}}
\newcommand{\Yla}[1]{Y_{\ell}({#1})}
\newcommand{\Ylm}[0]{Y^{(\ell)}_m}
\newcommand{\Ylma}[1]{Y^{(\ell_{#1})}_{m_{#1}}}
\newcommand{\Vl}[0]{\mathcal{V}_\ell}
\newcommand{\V}[1]{\mathcal{V}_{#1}}
\newcommand{\Lt}[1]{\mathbb{L}_2({#1})}
\newcommand{\li}[0]{\ell_i}
\newcommand{\lf}[0]{\ell_f}
\newcommand{\lo}[0]{\ell_o}
\newcommand{\ruv}[0]{r_{uv}}
\newcommand{\rhuv}[0]{\hat{r}_{uv}}
\newcommand{\Vin}[0]{V_\mathrm{in}}
\newcommand{\Vout}[0]{V_\mathrm{out}}
\newcommand{\cin}[1]{c^\mathrm{in}_{#1}}
\newcommand{\cout}[1]{c^\mathrm{out}_{#1}}
\newcommand{\kth}[0]{k^\mathrm{th}}
\newcommand{\pb}[0]{\mathbf{p}}
\newcommand{\xb}[0]{\mathbf{x}}
\newcommand{\vb}[0]{\mathbf{v}}
\newcommand{\hb}[0]{\mathbf{h}}
\newcommand{\zb}[0]{\mathbf{h}}
\newcommand{\lr}[1]{\left(#1\right)}
\newcommand{\set}[1]{\left\{#1\right\}}
\newcommand{\twovec}[2]{\left[\begin{array}{c}#1\\ #2\end{array}\right]}
\newcommand{\tworowvec}[2]{\left[\begin{array}{cc}#1 & #2\end{array}\right]}
\newcommand{\twomat}[4]{\left(\begin{array}{cc}#1 & #2\\ #3 &#4\end{array}\right)}
\newcommand{\threevec}[3]{\left[\begin{array}{c}#1\\ #2\\ #3\end{array}\right]}
\newcommand{\threerowvec}[3]{\left[\begin{array}{ccc}#1 & #2 & #3\end{array}\right]}
\newcommand{\threemat}[9]{\left(\begin{array}{ccc}#1 & #2& #3\\ #4 & #5 & #6\\ #7 & #8 & #9\end{array}\right)}
\newcommand{\Exp}[1]{\mathbb E\left[#1\right]}
\newcommand{\norm}[1]{\left|\left|#1\right|\right|}
\newcommand{\tr}{\mathrm{Tr}}
\newcommand{\trarg}[1]{\mathrm{Tr}\big(#1\big)}
\newcommand{\Ch}[0]{\widehat{\mathcal C}}
\newcommand{\PiTwo}[0]{\Pi_{\geq 2}}
\newcommand{\vset}{\mathcal V}
\newcommand{\eset}{\mathcal E}
\newcommand{\mpop}{\mathbf P}
\newcommand{\rb}[2]{\mathcal{B}^{(#1)}\big(#2\big)}
\newcommand{\pt}[1]{\mathcal{P}_t^{(#1)}}
\newcommand{\relu}{\texttt{ReLU}}
\renewcommand{\tanh}{\texttt{tanh}}
\newcommand{\identity}{\texttt{identity}}
\newcommand{\elu}{\texttt{ELU}}
\newcommand{\leakyrelu}{\texttt{LeakyReLU}}
\newcommand{\softmax}{\textsc{softmax}}
\renewcommand{\sb}{s_\beta}
\newcommand{\st}{s_t}
\newcommand{\trainset}{\mathcal T}
\newcommand{\attn}{\mathbfcal{A}}
\newcommand{\gat}{\tilde A_\mathrm{GAT}}
\newcommand{\graphtrans}{\tilde A_\mathrm{GT}}
\newcommand{\dadam}{\delta_\mathrm{adam}}
\newcommand{\mlp}[1]{\mathrm{MLP}\big(#1\big)}
\newcommand{\mhsa}[1]{\mathrm{MHSA}\big(#1\big)}
\newcommand{\gnn}[1]{\mathrm{GNN}\big(#1\big)}
\newcommand{\mpnn}[1]{\mathrm{MPNN}\big(#1\big)}
\newcommand{\LN}[1]{\text{LN}\big(#1\big)}
\newcommand{\venc}[1]{f^{(0)}_x(#1)}
\newcommand{\eenc}[1]{f^{(0)}_e(#1)}
\newcommand{\vdec}[1]{f^{(L+1)}_x(#1)}
\newcommand{\gdec}[1]{f^{(L+1)}_g\big(#1\big)}
\newcommand{\dec}[1]{f^{(L+1)}\big(#1\big)}
\newcommand{\betamsg}[0]{\beta_\mathrm{msg}}
\newcommand{\betaatt}[0]{\beta_\mathrm{att}}
\newcommand{\betamlp}[0]{\beta_\mathrm{mlp}}
\newcommand{\tepoch}[0]{\tau_\mathrm{epoch}}
\newcommand{\titer}[0]{\tau_\mathrm{iter}}
\newcommand{\one}[0]{\mathbf{1}}
\newcommand{\x}[2]{\mathbf{x}^{(#1)}_{#2}}
\newcommand{\X}[1]{\mathbf{X}^{(#1)}}
\newcommand{\vecX}[1]{\mathrm{vec}\big(\X{#1}\big)}
\newcommand{\E}[1]{\mathbf{E}^{(#1)}}
\newcommand{\e}[1]{\mathbf{e}^{(#1)}}
\newcommand{\A}[0]{\mathbf{A}}
\newcommand{\ntrain}[0]{N_\mathrm{train}}
\newcommand{\W}[2]{\mathbf W^{(#1)}_{#2}}
\newcommand{\btheta}[1]{\boldsymbol{\theta}^{#1}}
\newcommand{\J}[4]{\boldsymbol{J}^{(#1,#3)}_{#2\leftarrow #4}}
\newcommand{\Jt}[2]{\boldsymbol{J}^{(#1, #2)}}
\newcommand{\ID}[0]{\mathbf{I}}
\newcommand{\grad}[1]{\boldsymbol{g}^{(#1)}}
\newcommand{\Xt}[1]{\tilde{\mathbf{X}}^{(#1)}}
\newcommand{\Wt}[1]{\tilde{\mathbf{W}}^{(#1)}}
\newcommand{\etasgd}[0]{\eta_{{}_\mathrm{GD}}}
\newcommand{\etaadam}[0]{\eta_{{}_\mathrm{Adam}}}
\newcommand{\Mab}[0]{M_{\alpha\beta}}
\newcommand{\Nba}[0]{N_{\mathcal{B}_\alpha}}
\newcommand{\Nbb}[0]{N_{\mathcal{B}_\beta}}
\newcommand{\Cab}[0]{C_{\alpha\beta}}

\newtheorem{theorem}{Theorem}
\newtheorem{lemma}[theorem]{Lemma}
\newtheorem{corollary}[theorem]{Corollary}
\newtheorem{proposition}[theorem]{Proposition}
\newtheorem{question}[theorem]{Question}
\newtheorem{conjecture}[theorem]{Conjecture}
\newtheorem{prob}{Problem}
\newtheorem{definition}[theorem]{Definition}
\newtheorem{remark}[theorem]{Remark}
\newtheorem{assumption}[theorem]{Assumption}

\maketitle

\begin{abstract}
The performance of deep learning models crucially depends on the settings of hyperparameters like learning rate, initialization scale, and weight decay. 
Hyperparameter transfer aims to make near-optimal hyperparameter settings consistent across model scale, so that large models can be optimized by proxy tuning their smaller, cheaper-to-optimize counterparts. 
While transfer principles are well-studied in the context of dense neural networks in language and vision tasks, they remain comparatively underexplored for graph neural networks (GNNs). 
We develop and validate a transfer parameterization for GNNs trained with SGD, Adam, and AdamW.
Through theoretical scaling analyses and controlled experiments, we show that the proposed parameterization yields stable feature updates, learning rate transfer, and improved performance as width and depth increase. 
For SGD, we identify graph-dependent first-layer correction factors and show that their use can accelerate early training in graphs with sparse bag-of-words inputs. 
For Adam, we explore how different message passing normalizations affect early- and late-training transfer behavior, illustrating the importance of message passing normalization and advocating for an associated hyperparameter. 
For AdamW, we adapt a parameterization that allows for the joint transfer of weight decay and learning rate. 
Together, these results provide a practical recipe for scaling GNNs across a variety of learning tasks and training scenarios. 
\end{abstract}

\section{Introduction}
\label{sec:introduction}

The performance of trained deep learning models is sensitive to the specification of various hyperparameters, most importantly initialization scale and learning rate.  Tuning these hyperparameters is often prohibitively expensive at large scale, i.e. in models with large width ($D$) and depth ($L$).  This has motivated the modern paradigm of hyperparameter transfer, a prescription for how to infer near-optimal hyperparameters in large models from near-optimal hyperparameters in much smaller models, where they are relatively cheap to find empirically. Hyperparameter transfer has been established for feedforward neural networks (NNs) \cite{yang2021zero}, residual networks \cite{bordelon2024depthwise}, dense transformers \cite{dey2025dontlazycompletepenables}, sparse MoE transformers \cite{jiang2026hyperparameter} and other models. Existing approaches to hyperparameter transfer seek a \textit{parameterization}, a prescription for how to modify hyperparameters and initialization schemes across model scale. Successful parameterizations typically satisfy two heuristics: 
\begin{enumerate}
    \item \textit{Maximal Updates}: The optimal setting of the model's hyperparameters leads to order-one updates of each neuron pre-activation from each step of optimization, independent of model size. 
    \item \textit{Limiting Behavior}: As model size increases, performance increases and training dynamics converge to a well-defined and fully non-linear limit.
\end{enumerate}

Though hyperparameter transfer has been extensively studied in neural networks (NNs) used in language and vision tasks, it remains relatively poorly understood in graph neural networks (GNNs). This is in part because many common GNNs suffer from GNN-specific numerical pathologies, including poor initialization, oversmoothing, and oversquashing~\cite{li2018deeper, oono2020graph, alon2021on}. These difficulties have motivated the development of novel GNN architectures including graph transformers (GTs)~\cite{shirzad2023exphormer,ladislav2022graphgps}, a broad class of models that apply restricted attention and static message passing to graph data.  

Our objective in this article is to extend and validate hyperparameter transfer principles for GNNs, establishing a foundation for future extensions to graph transformers. Our key contributions are the following: 
\begin{itemize}
    \item We derive and empirically validate learning rate transfer scalings for GNNs trained with SGD (Props.~\ref{th:feature_stability} and~\ref{th:update-stability-gd-lr-graph-regression}, Fig.~\ref{fig:sweep-lr_sgd})  and Adam (Prop.~\ref{th:adam_lrs}, Fig.~\ref{fig:sweep-lr_adam}), summarized in Eqn.~\ref{eq:global_lrs}.
    \item We show that for GD on graph regression, correlations between node features belonging to different graphs change the scale of encoded representations (Prop.~\ref{th:update-stability-gd-lr-graph-regression}), necessitating the use of a first-layer learning rate correction factor for some datasets. We show that the form of the correction factor is slightly altered when batches of graphs are used (Remark~\ref{th:batching-effects_sgd-graph-regression}) or when the task is node-level (Remark~\ref{th:batching-effects_gd-vertex-classification}). 
    Empirically, we show that the first-layer correction factor is important for speeding up early training in citation networks with sparse bag-of-words features (Figs.~\ref{fig:first-layer-corr} and~\ref{fig:flc-epochs-heatmap}, Sec.~\ref{sec:additional-fl-corr-study}). 
    \item We adapt the AdamW weight decay parameterization recipe from Ref.~\cite{wang2025setadamwsweightdecay} and demonstrate learning rate transfer for GNNs trained with AdamW (Fig.~\ref{fig:sweep-lr_adamw}) with the correct weight decay parameterization (Eqns.~\ref{eq:weight_decay},~\ref{eq:t_epoch}).  
    \item We explore the effect of using message passing operators with different normalizations in Figs.~\ref{fig:normalization-sweep_pascal} and~\ref{fig:normalization-sweep_mnist}, finding that normalized operators are important for robust transfer in early and late training, and advocating for a global normalization strategy controlled by a hyperparameter $\gamma$ (Eqn.~\ref{eq:normalized-gcn-update}). 
\end{itemize}

\begin{figure}[t]
    \centering
    \includegraphics[width=0.7\textwidth]{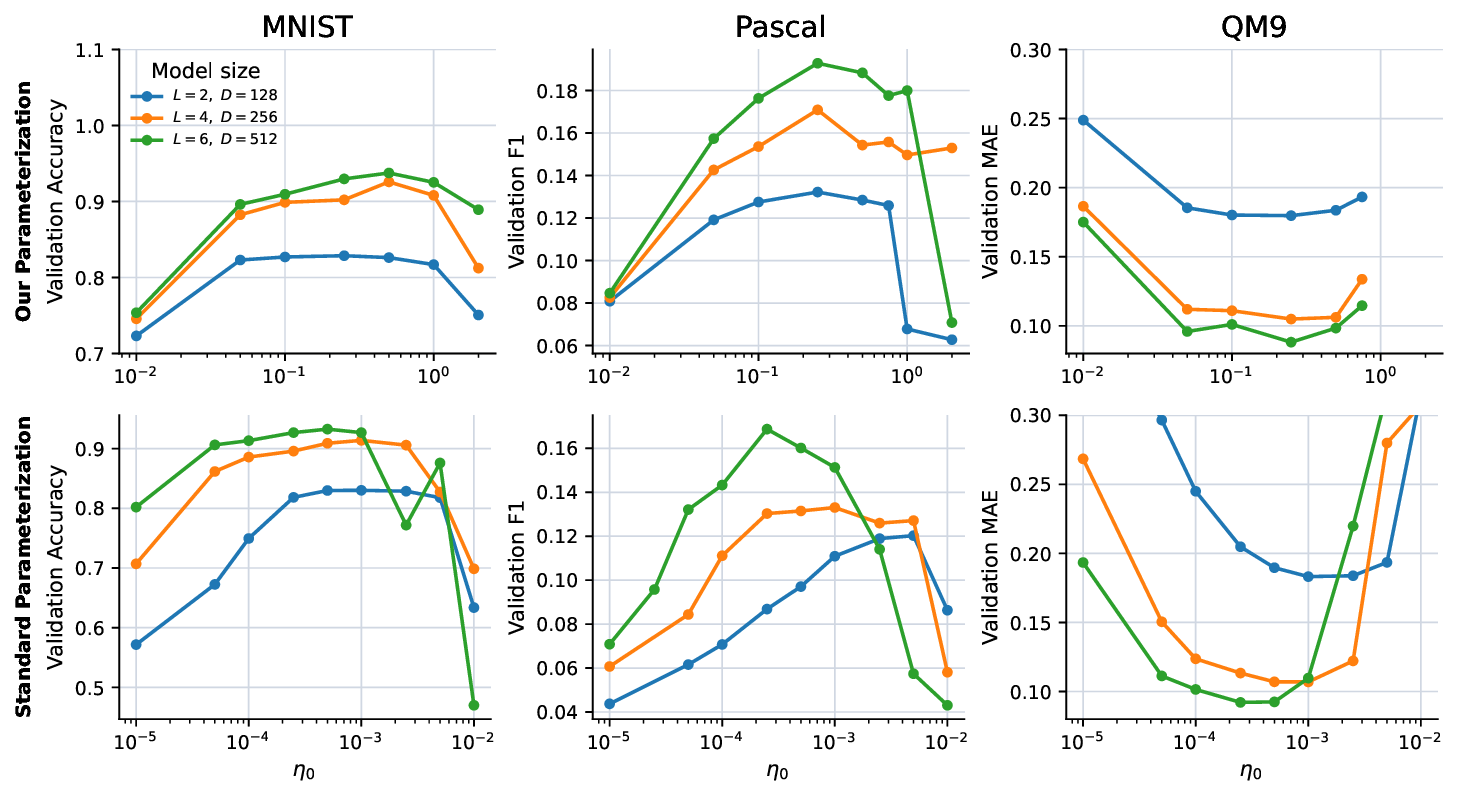}
    \caption{
    Learning rate transfer under 
    our parameterization (top) and the standard parameterization (bottom; see Table~\ref{tab:our_vs_standard_param}) for GNNs with degree-normalized message passing operators trained via Adam. 
    }
    \label{fig:sweep-lr_adam}
\end{figure}

\begin{figure}[t]
    \centering
    \includegraphics[width=0.7\textwidth]{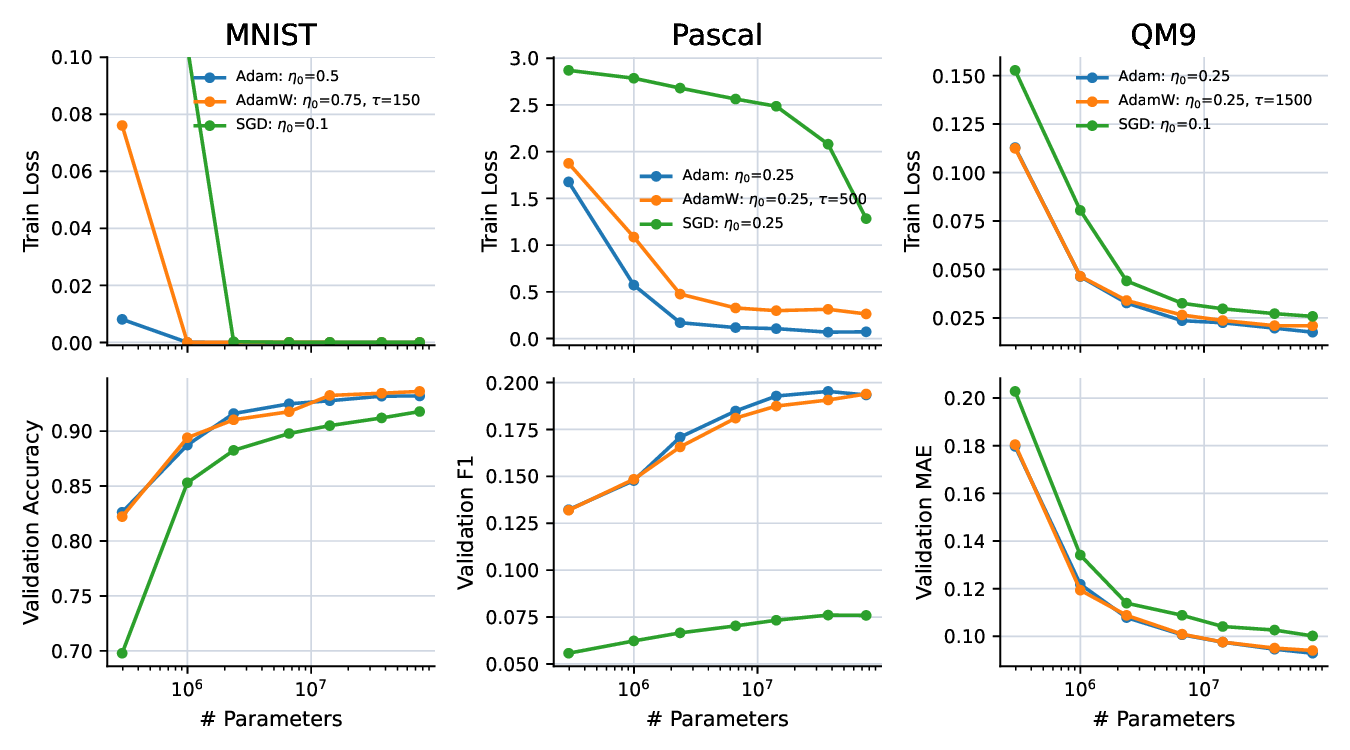}
    \caption{
        Performance increases with scale at fixed near-optimal values of $\eta_0$.
    }
    \label{fig:sweep-size}
\end{figure}

\section{Summary of Results}
\label{sec:summary-of-findings}

In the following sections, we summarize our empirical results and connect them to the scaling analysis presented in Sec.~\ref{sec:scaling_analysis}. 
In Sec.~\ref{sec:graph-notation} we establish our notation for graphs and GNNs, and in Sec.~\ref{sec:base-transfer-gnn} we apply it to specify a base transfer GNN and its optimizer-specific initialization scales and learning rate parameterizations. 
Sec.~\ref{sec:learning-rate-transfer-SGD} demonstrates learning rate transfer with SGD and additionally explores the use of the data-dependent first-layer correction factors derived in Prop.~\ref{th:update-stability-gd-lr-graph-regression}. 
Sec.~\ref{sec:learning-rate-transfer-adam} demonstrates learning rate transfer with Adam and AdamW, the latter of which requires a width-dependent parameterization for weight decay. 
Finally, in Sec.~\ref{sec:adjacency-normalization} we explore the effect of the normalization of the message passing operator on learning rate transfer.

\subsection{Graph Notation and Datasets}
\label{sec:graph-notation}

Consider a graph $\mathcal G = (\vset, \eset)$ with $N:=|\vset|$ nodes and $M:=|\eset|$ edges.
In most graph datasets, nodes have features $\mathbf{X}\in\R^{N\times n_0}$ and, in some, edge features $\mathbf{E}\in\R^{M\times m_0}$.
Occasionally, we will refer to the features belonging to vertex $u$ as $\mathbf{x}_u\in\R^{1\times n_0}$.
Graphs are equipped with adjacency matrices $\A\in\{0,1\}^{N\times N}$ summarizing their local connectivity,
which we take to be symmetric in our experiments. 
It is common to use normalized adjacency matrices in GNNs, for example the symmetrically degree-normalized adjacency matrix $\tilde{\A}:= \mathbf D^{-\frac{1}{2}}\mathbf A\mathbf D^{-\frac{1}{2}}$ where $\mathbf D$ is the diagonal degree matrix. 
One simple GNN variant is the graph convolutional network (GCN, see Ref.~\cite{kipf2017semisupervised}), which passes messages weighted by the entries of $\tilde{\A}$, which is canonically taken to include self-loops, 
\begin{equation}
    \x{\ell+1}{u} 
    = 
    \frac{1}{\sqrt{D}}
    \sum_{v\in\mathcal{N}(u)}
    \frac{1}{\sqrt{d_ud_v}}
    \x{\ell}{v}
    \W{\ell+1}{}
    \qquad
    \leftrightarrow
    \qquad
    \X{\ell+1} 
    = 
    \frac{1}{\sqrt{D}} 
    \tilde{\A}
    \X{\ell}
    \W{\ell+1}{}
\label{eq:gcn-forward-pass}
\end{equation}
where $\ell$ indexes the message passing depth $L$ and $D$ is the width. 
Our goal will be to scale GNNs with respect to both depth and width; in our experiments, we scale these parameters jointly.  
The forward pass is scaled by $D^{-1/2}$ to compensate for weights $\W{\ell+1}{}\in\R^{D\times D}$ initialized iid from $\mathcal{N}(0,1)$. 
Throughout, we will refer to generic message passing operators $\mathbf{P}\in\R^{N\times N}$ that may be normalized (like $\tilde{\A}$) or un-normalized (like $\A$). 
In this case, our generalized GNN forward pass is
\begin{equation}
    \mathbf{X}^{(\ell+1)} = \frac{1}{\sqrt{D}}\mathbf{P}\X{\ell}\W{\ell+1}{}.
    \label{eq:generic-gnn-forward-pass}
\end{equation}
Throughout we experiment on a collection of graph benchmarks spanning graph-level classification (MNIST Superpixels classification~\cite{lecun-mnisthandwrittendigit-2010}) and regression (QM9 dipole moment prediction~\cite{ramakrishnan2014qm9}), and node-classification (PascalVOC-SP image segmentation~\cite{dwivedi2022lrgb}). 
We do not use edge features, though they are present in some datasets. 
In Sec.~\ref{sec:learning-rate-transfer-SGD} we provide an additional experiment with the Cora, CiteSeer, and PubMed semi-supervised citation classification datasets~\cite{yang2016planetoid}. 
Full benchmark dataset details are provided in Sec.~\ref{sec:benchmark-datasets}. 
The full training details used to produce each figure are listed in Sec.~\ref{sec:run-specifications}.

\subsection{Base Transfer GNN}
\label{sec:base-transfer-gnn}

GNNs are closely related to transformers, which compute attention-weighted messages across a fully-connected ``graph'' on the inputs. 
The key difference between GNNs and transformers is that the edge connectivity, a relational inductive bias encoded in $\A$, appears directly in the forward pass. 
Nonetheless, we work in analogy with hyperparameter transfer principles developed for transformers (see Sec.~\ref{sec:review-of-literature}), and find that they translate well to the graph domain. 
Modern transformer architectures are based on the successive application of an attention-weighted, all-to-all message passing layer and a dense feedforward layer.
In Refs.~\cite{bordelon2024transformer,dey2025dontlazycompletepenables}, it was shown that CompleteP hyperparameter transfer crucially hinges on the use of residual connections scaled by $1/L$, where $L$ is the depth of message passing, in each layer. 
In analogy with these results, we propose the following \textit{base transfer GNN}:
\begin{align}
    &\X{0} = 
    \frac{1}{\sigma_0}
    \venc{\mathbf X, \A}
    \label{eq:base-model-encoder}
    \\
    &\Xt{\ell+1}
    = 
    \X{\ell} 
    +
    \frac{1}{L}\ 
    \text{MPNN}^{(\ell+1)}\big(
        \X{\ell}, \A
    \big)
    \label{eq:base-model-mpnn}
    \\
    &\X{\ell+1} = 
    \Xt{\ell+1}
    + 
    \frac{1}{L}\ 
    \text{MLP}^{(\ell+1)}\big(
        \Xt{\ell+1}
    \big)
    \label{eq:base-model-mlp}
    \\
    &\mathbf{Z} = \frac{1}{\sigma_{L+1}\sqrt{D}} \vdec{\X{L}}
    \label{eq:base-model-decoder}
\end{align}
where the layer-wise index $\ell$ runs from $0$ to $L-1$ in the residual branches in Eqns.~\ref{eq:base-model-mpnn}-\ref{eq:base-model-mlp}. 

\setlength{\tabcolsep}{8pt} 
\renewcommand{\arraystretch}{1.5} 
\begin{table}[htb!]
    \caption{Required per-layer learning rates $\eta_\ell$ and settings of $\sigma_\ell$ for using global learning rates. }
  \centering
  \begin{tabular}{|ll|ll|ll|}
    \hline
     & & \textbf{SGD} &  & \textbf{Adam} &
     \\
     & $\ell$ & $\eta_\ell$ & $\sigma_\ell$ & $\eta_\ell$ & $\sigma_\ell$
    \\
    \hline
    \textbf{Encoder} & $0$ & $\eta_0 D\sigma_0^2 $ & $\sqrt{L}$ & $\eta_0\sigma_0$ & $1/\sqrt{D}$
    \\
    \textbf{Residual} & $[1:L]$ & $\eta_0 DL$ & 1 & $\eta_0/\sqrt{D}$ & 1
    \\
    \textbf{Decoder} & $L+1$ & $\eta_0 D\sigma_{L+1}^2 $ & $\sqrt{L}$ & $\eta_0\sigma_{L+1}$ & $1/\sqrt{D}$
    \\
    \hline
  \end{tabular}
  \vspace{2mm}
  \label{tab:learning_rate_settings}
\end{table}

The encoder module $\venc{\cdot}$ in Eqn.~\ref{eq:base-model-encoder} may be a simple linear layer or a more complicated structure leveraging positional encodings, hence the dependence on $\A$.
The factor of $\sigma_0^{-1}$ is introduced to control the effective learning rate in the encoder layer, whose weights should be drawn iid from $\mathcal{N}(0,\sigma_0^2)$. 
In the forward pass, $\sigma_0^{-1}$ simply cancels the upscaled initialization variance, but in the backward pass it survives to rescale the gradient. 
As we describe below, this allows us to fix an \textit{effective learning rate} in the encoder layers while using a \textit{global learning rate} shared across all layers. 

The residual stream in Eqns.~\ref{eq:base-model-mpnn}-\ref{eq:base-model-mlp} is computed in two successive residual steps; the first uses a message passing NN (MPNN) and the second uses a feedforward NN (MLP). 
The successive application of local neighborhood updates and dense NN updates are inspired by the transformer architecture. 
The factors of $1/L$ in Eqns~\ref{eq:base-model-mpnn}-\ref{eq:base-model-mlp} ensure the gradient in the residual layers scales like $1/L$, so that the total contribution of the residual stream to the network remains order-one. 

The decoder module $\vdec{\cdot}$ in Eqn.~\ref{eq:base-model-decoder} may be a simple linear layer, though deeper decoder layers are possible. 
As in the encoder layer, a factor of $\sigma_{L+1}$ is introduced to control the gradient scale, and hence the effective learning rate, in the decoder layer. 
All decoder weights should be initialized iid from $\mathcal{N}(0,\sigma_{L+1}^2)$. 
An additional factor of $1/\sqrt{D}$ is used to keep the width scaling of the decoder gradients the same as the width scaling in the residual gradients; though this suppresses the scale of the output logits, their updates will remain order-one. 

Each of these encoder, MPNN, MLP, and decoder modules can be chosen flexibly, so long as they do not significantly grow or shrink their inputs as $D$ increases (in which case we will refer to them as \textit{order-one modules}).

Crucially, to achieve hyperparameter transfer each of these layers needs a different effective learning rate $\eta_\ell$. 
We choose to use global learning rates, applied equally to all layers, and use rescaling factors $\sigma_\ell$ to adjust the effective learning rates in each layer as necessary.
By fixing the global learning rates to be those required in the residual layers, we need only use rescaling factors $\sigma_0$ and $\sigma_{L+1}$ in the encoder and decoder layers respectively. 

To update a model parameter $\theta$, we will use the optimization schemes 
\begin{equation}
    \Delta\theta\vert_\text{SGD}
    = 
    -\etasgd \nabla_\theta \mathcal L,
    \qquad
    \text{and}
    \qquad
    \Delta\theta\vert_\text{Adam} = -\etaadam \frac{\hat{m}_\theta}{\sqrt{\hat{v}_\theta} + \epsilon}
    \label{eq:optimizers}
\end{equation}
where for Adam $\hat{m}_\theta$ and $\hat{v}_\theta$ are the first and second bias-corrected gradient moments and $\epsilon$ is a small numerical constant. 
Our global learning rate parameterizations are 
\begin{equation}
    \etasgd = \eta_0DL
    \qquad
    \text{and}
    \qquad 
    \etaadam = \eta_0/\sqrt{D},
\label{eq:global_lrs}
\end{equation}

and the effective learning rates induced by the choices of $\sigma_\ell$ are shown in Table~\ref{tab:learning_rate_settings}. 
A full theoretical analysis of these scalings is given in Sec.~\ref{sec:scaling_analysis}.

\subsection{Learning Rate Transfer with SGD}
\label{sec:learning-rate-transfer-SGD}

In Prop.~\ref{th:update-stability-gd-lr-graph-regression}, we analyze one step of GD for 1D graph regression to argue that the per-layer learning rates must be 
\begin{equation*}
    \etasgd^{(0)} = \eta_0D\sigma_0^2\times \Cab,
    \qquad 
    \etasgd^{(1\leq\ell\leq L)} = \eta_0DL,
    \qquad
    \text{and}
    \qquad
    \etasgd^{(L+1)} = \eta_0D\sigma_{L+1}^2,
\end{equation*}
where to enact a global learning rate we chose $\sigma_0=\sigma_{L+1}=\sqrt{L}$, and $\Cab$ is a first-layer learning rate correction that arises to account for the differences between an initial graph $\mathcal{G}_\alpha$ used to update the weights under GD, and a new graph $\mathcal{G}_\beta$ that the updated model makes a prediction on.
If the graphs have numbers of nodes $N_\alpha$ and $N_\beta$, and node features $\mathbf{X}_\alpha$ and $\mathbf{X}_\beta$, then this one-step analysis shows that
\begin{equation}
    \Cab
    :=
    \frac{n_0\sqrt{N_\beta}}{M_{\alpha\beta}}
    \qquad
    \text{with}
    \qquad 
    M_{\alpha\beta}
    :=
    \norm{
        \frac{1}{N_\alpha}\one^T_{N_\alpha}
        \mathbf X_\alpha
        (\mathbf X_\beta)^T
    }_2, 
\label{eq:first-layer-correction}
\end{equation}
though experimentally we use an empirical average over many graphs, 
\begin{equation}
    \widehat{C}_{\alpha\beta}
    :=
    \mathrm{MEAN}_{\alpha,\beta}\bigg(
        \frac{n_0\sqrt{N_\beta}}{M_{\alpha\beta}}
    \bigg).
    \label{eq:empirical-Cab}
\end{equation}
Note that the form of $\Cab$ changes slightly if graph batches are used (Remark~\ref{th:batching-effects_sgd-graph-regression}) or if the task requires node-level outputs (Remark~\ref{th:batching-effects_gd-vertex-classification}), but in practice it is well-approximated by the forms above.
The quantity $\Mab$ crucially depends on dot products between the mean node feature on all of $\mathcal{G}_\alpha$ and the node features of $\mathcal{G}_\beta$. 
When the features are approximately aligned, therefore, $\Mab$ will be close to $n_0\sqrt{N_\beta}$.
However, if the features are sparse or one-hot, $\Mab$ will scale more like $\sqrt{N_\beta}/n_0$. 
For this reason, we emphasize that $\Cab$ will vary based on the specifics of the graph inputs.

In Table~\ref{tab:cab-per-dataset} we list the corresponding $\widehat{C}_{\alpha\beta}$ values for the graph datasets explored in this paper.
We also include an empirical sparsity measure $\hat{s}_X$, defined as the per-graph sparsities
\begin{equation*}
    s_X := 1 - \frac{||\mathbf{X}||_0}{Nn_0}
\end{equation*}
averaged over the corresponding dataset, where the $\ell_0$ norm of the node feature matrices $\mathbf{X}\in\R^{N\times n_0}$ counts the number of non-zero entries, 
\[
    \norm{\mathbf{X}}_0 =
    |\{(u,i)\ :\ \mathbf{X}_{u,i}\neq 0 \}|. 
\]
We observe that datasets with very sparse features, particularly citation networks with bag-of-words features drawn from large dictionaries, have higher values of $\widehat{C}_{\alpha\beta}$; note, however, that high sparsity does not guarantee features will be aligned or anti-aligned. 

\begin{table}[htb!]
    \caption{
        Estimated sparsity ($s_X$), correlation magnitude $\widehat{M}_{\alpha\beta}$, and gradient descent first-layer correction factors $\widehat{C}_{\alpha\beta}$ per dataset.
        Averages over 50 graphs are reported for MNIST-SP, PascalVOC-SP, and QM9. Cora, CiteSeer, and PubMed are single graph datasets whose $N$ corresponds to the number of train nodes.  
      }
  \centering
  \setlength{\tabcolsep}{10pt}
  \renewcommand{\arraystretch}{1.15}
  \begin{tabular}{lccccc}
    \toprule
    Dataset & $N$ & $n_0$ &  $s_X$ & $\widehat{M}_{\alpha\beta}$ & $\widehat{C}_{\alpha\beta}$ \\
    \midrule
    MNIST-SP & 75 & 1 & 0.67 & 1.7  & 5.2 \\
    PascalVOC-SP & 479 & 14 & 0.04 & 259  & 1.2 \\
    QM9 & 18 & 11 & 0.79 & 38.4 & 1.2 \\
    Cora & 1208 & 1433 & 0.99 & 3005 & 16.7 \\
    CiteSeer & 1827 & 3703 & 0.99 & 7203 & 22.0 \\
    PubMed & 18217 & 500 & 0.90 & 4990 & 13.5 \\
    \bottomrule
  \end{tabular}
  \label{tab:cab-per-dataset}
\end{table}
 
For MNIST-SP, PascalVOC-SP, and QM9, we find $\widehat{C}_{\alpha\beta}$ is order-one; we suppress it and show that we can still achieve excellent learning rate transfer via SGD in Fig.~\ref{fig:sweep-lr_sgd}. 
Note that for stability in late training via SGD, we have used LayerNorm~\cite{ba2016layernormalization} with no trainable parameters on the inputs to the MPNN and MLP modules in Eqns.~\ref{eq:base-model-mpnn} and~\ref{eq:base-model-mlp}. 
We observe that without LayerNorm, models trained via SGD achieve stable performance during convergence, but eventually destabilize. 
In Fig.~\ref{fig:sweep-size}, we show how models of increasing size perform under SGD at the optimal values of $\eta_0$ identified in the transfer scan. 

For citation networks like Cora, CiteSeer, and PubMed with sparse bag-of-words features, we find much larger values of $\widehat{C}_{\alpha\beta}$, where first-layer learning rate corrections may be significant. 
In Fig.~\ref{fig:first-layer-corr}, we show that using first-layer learning rate corrections to SGD without LayerNorm, at near-optimal values of $\eta_0$ (as determined by an auxiliary scan), significantly speeds up the training time for these citation network datasets. 
Note that the form of $\Cab$ suppresses effects at and above $L^{-2}$ (see Prop.~\ref{th:update-stability-gd-lr-graph-regression}), such that it should be treated as an approximate correction only. 
Sec.~\ref{sec:additional-fl-corr-study} explores a first-layer correction factor $C$ as a proxy for $\Cab$, finding that for the citation networks CiteSeer and PubMed, $C\gg 1$ yields a broad basin of fast convergence over $\eta_0$, whereas $C\sim 1$ is markedly slower.

\begin{figure}[t]
\begin{centering}
    \includegraphics[width=0.7\textwidth]{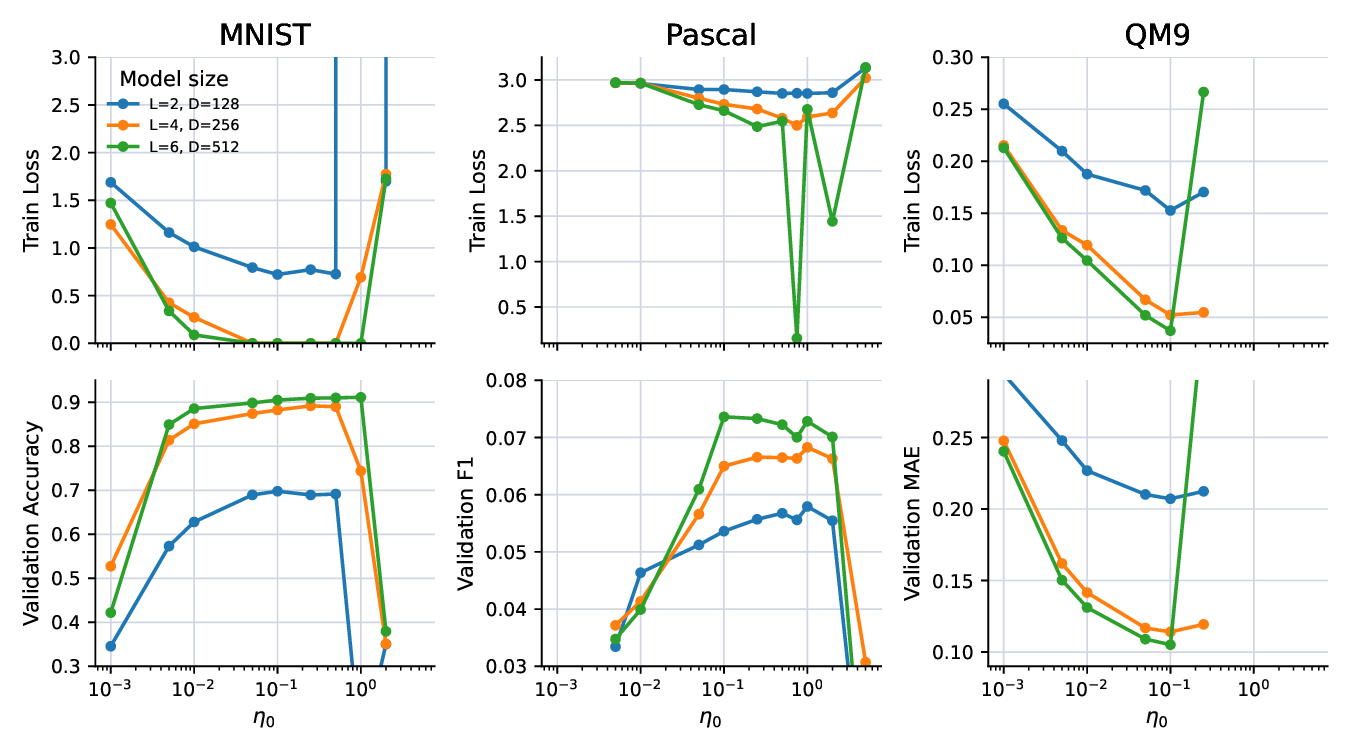}
    \caption{Learning rate transfer for GNNs trained via SGD. }
    \label{fig:sweep-lr_sgd}
\end{centering}
\end{figure}

\begin{figure}[t]
\begin{centering}
    \includegraphics[width=0.7\textwidth]{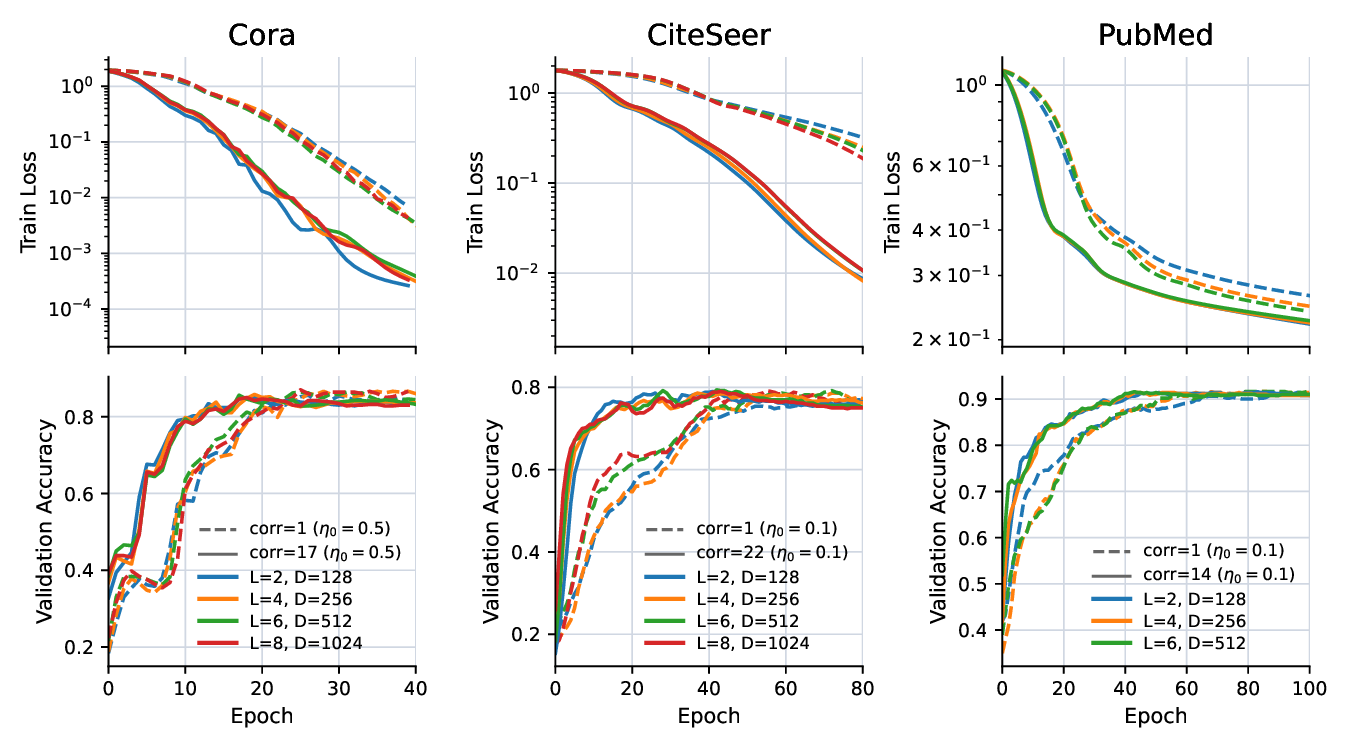}
    \caption{GNNs trained via GD with (solid) and without (dashed) first-layer learning rate corrections on semi-supervised node classification datasets at optimal $\eta_0$ values determined in an auxiliary learning rate scan.}
    \label{fig:first-layer-corr}
\end{centering}
\end{figure}

\subsection{Learning Rate Transfer with Adam and AdamW}
\label{sec:learning-rate-transfer-adam}

We recover the Adam learning rates in Prop.~\ref{th:adam_lrs}, showing that per-layer they are
\[
    \etaadam^{(0)}=\eta_0\sigma_0,
    \qquad
    \etaadam^{(1\leq \ell\leq L)} = \frac{\eta_0}{\sqrt{D}},
    \qquad
    \text{and}
    \qquad
    \etaadam^{(L+1)} = \eta_0\sigma_{L+1},
\]
so that to use a global learning rate we choose $\sigma_0=\sigma_{L+1}=1/\sqrt{D}$. 
We demonstrate learning rate transfer for Adam in Fig.~\ref{fig:sweep-lr_adam} and, at the optimal values of $\eta_0$ identified in a transfer scan, show how models of increasing size perform in Fig.~\ref{fig:sweep-size}. 
Fig.~\ref{fig:sweep-lr_adam} additionally compares against the standard parameterization described in Sec.~\ref{sec:standard-param-baseline}. 
Our parameterization preserves a common near-optimal base learning rate across model scale, whereas the standard parameterization optima shift towards smaller learning rates as size increases. 

It was recently observed that parameter updates via AdamW~\citep{loshchilov2017adamw}, a variant of Adam with weight decay of strength $\lambda$ on each trainable parameter, can be viewed as an exponential moving average (EMA) with iteration timescale $\titer=(\lambda\eta)^{-1}$~\cite{wang2025setadamwsweightdecay}.
For a trainable parameter $\theta$, the EMA view of the AdamW update is 
\[ 
    \theta_{t+1} = (1-\titer^{-1})\theta_{t} + \titer^{-1} q_t
\]
where $q_t:=-\frac{1}{\lambda}\frac{\hat m_t}{\sqrt{\hat v_t}+\epsilon}$ with empirical first and second gradient moments $\hat{m}_t$ and $\hat{v}_t$ respectively, with a small constant $\epsilon$.
Importantly, there is evidence that $\epsilon$ may need to scale with model size~\cite{dey2025dontlazycompletepenables}; throughout our experiments we suppress its effect by setting it to $10^{-14}$. 
$\titer$ transfers as long as the product $\lambda\eta$ is invariant with respect to model size at fixed batch and dataset sizes; since $\eta=\eta_0/\sqrt{D}$ for Adam, weight decay should be parameterized as
\begin{equation}
    \lambda = \lambda_0\sqrt{D}
\label{eq:weight_decay}
\end{equation}
for some tunable base weight decay $\lambda_0$. 
It is convenient to rephrase $\titer$ in units of epochs. 
Given $\ntrain$ training examples and a batch size of $B$, there are approximately $M:=\ntrain/B$ gradient steps per epoch. Then, the iteration timescale in units of epochs is
\begin{equation}
   \tepoch := \frac{\titer}{M} =
    \frac{1}{\lambda_0\eta_0M}
    = 
    \frac{B}{\lambda_0\eta_0\ntrain}.
\label{eq:t_epoch}
\end{equation}
That is, for a given dataset with $\ntrain$ examples and fixed $\eta_0$ and $B$, the quantity $\tepoch$ (equivalently $\lambda_0\propto \tepoch^{-1}$) should transfer. 
In line with Ref.~\cite{wang2025setadamwsweightdecay}, we observe in Fig.~\ref{fig:tune-t-epoch} that the reasonable values of $\tepoch$ are consistent across a range of learning rates. 
Values of $\tepoch$ identified in this scan are fixed and used to perform an AdamW learning rate transfer scan in Fig.~\ref{fig:sweep-lr_adamw}. 
There, for non-negligible weight decay, we observe excellent transfer, and use the optimal $\eta_0$ values identified in the sweep to train models of increasing size in Fig.~\ref{fig:sweep-size}, demonstrating that performance improves with scale for AdamW in this configuration. 

\begin{figure}[t]
\begin{centering}
    \includegraphics[width=0.7\textwidth]{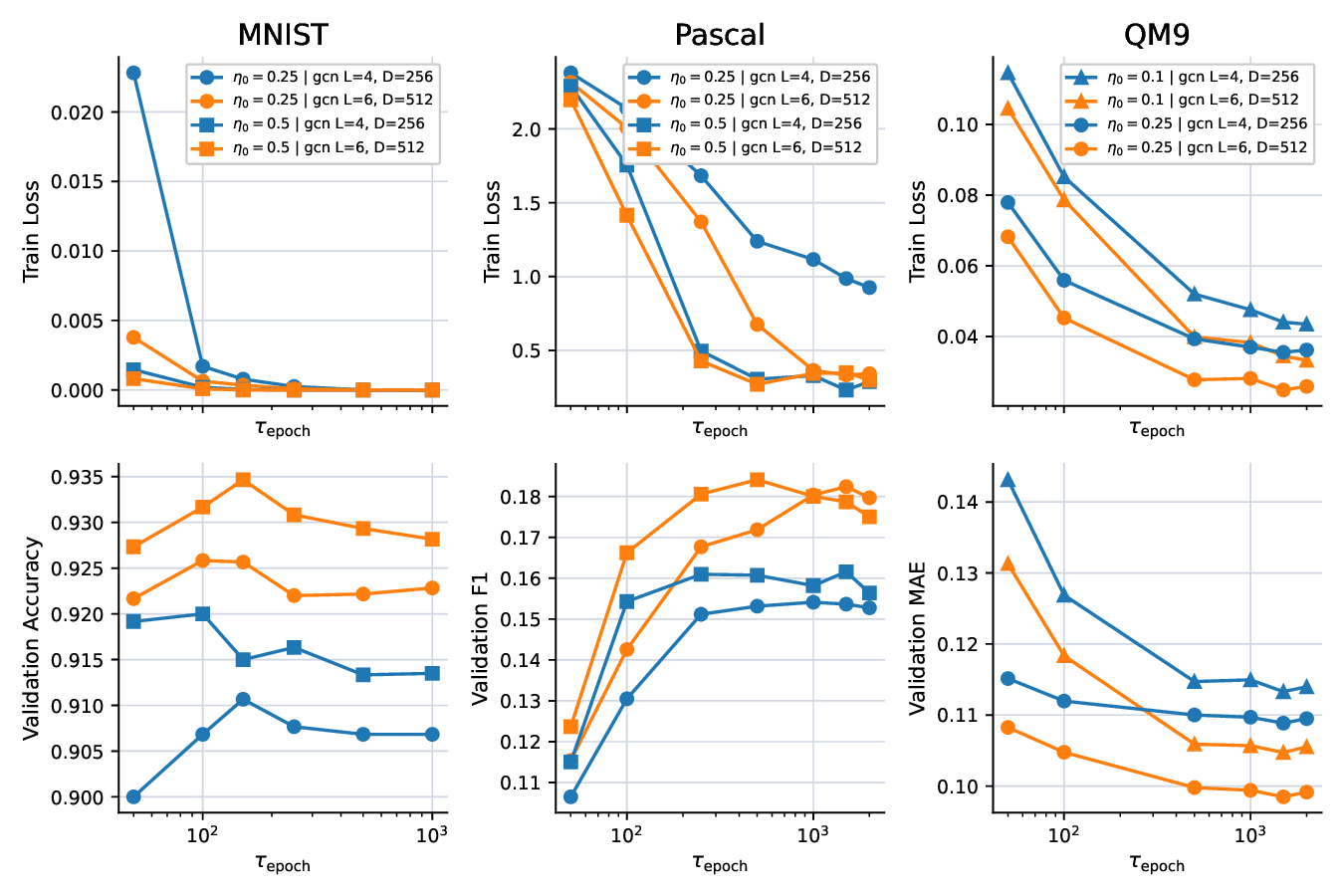}
    \caption{
    Weight decay transfer via $\tepoch$ for GNNs trained with AdamW at near-optimal $\eta_0$. 
    }
    \label{fig:tune-t-epoch}
\end{centering}
\end{figure}

\begin{figure}[t]
\begin{centering}
    \includegraphics[width=0.7\textwidth]{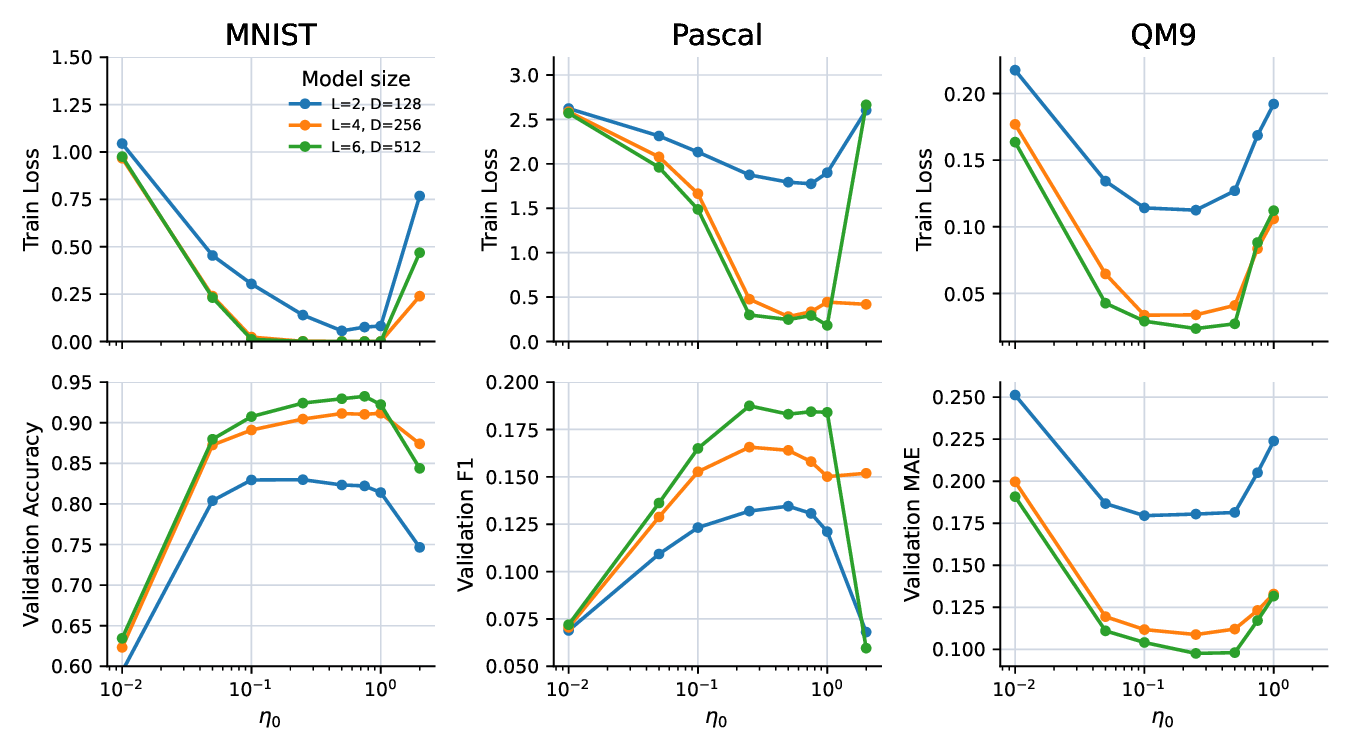}
    \caption{
    Learning rate transfer for GNNs trained with AdamW at fixed values of $\tepoch$. 
    }
    \label{fig:sweep-lr_adamw}
\end{centering}
\end{figure}

\subsection{Normalization of the Message Passing Operator}
\label{sec:adjacency-normalization}

We have taken $\mpop=\tilde{\A}$ (see Eqn.~\ref{eq:generic-gnn-forward-pass}) in all experiments so far, so that message passing is symmetrically degree-weighted mean aggregation over each graph neighborhood. 
This normalization is important to achieve the numerical stability required for hyperparameter transfer; to illustrate, consider the shift in scale induced by $\mpop$ in a simple GNN forward pass when the weight entries are initialized iid from $\mathcal{N}(0,1)$,
\[
    \Exp{
        \norm{
            \frac{1}{\sqrt{D}}
            \mathbf{P}
            \X{\ell}
            \W{\ell+1}{}
        }_F^2
    } 
    = 
    \Exp{
        \norm{
            \mathbf{P}
            \X{\ell}
        }_F^2
    } 
    :=
    \gamma_\ell^2
    \ 
    \Exp{
        \norm{
            \X{\ell}
        }_F^2
    },
\]
where we have defined the shift factor 
\begin{equation}
    \gamma_\ell^2
    := 
    \frac{
        \Exp{
            \norm{
                    \mathbf{P}
                \X{\ell}
            }_F^2
        }
    }{
        \Exp{
            \norm{
                \X{\ell}
            }_F^2
        }
    }. 
    \label{eq:gamma-def}
\end{equation}
An unnormalized operator $\mathbf{P}=\A$, which facilitates sum aggregation across the graph as in, e.g., the graph isomorphism network (GIN)~\cite{xu2018gin}, will induce a larger $\gamma_\ell$ than a normalized message passing operator.
To mitigate this change in scale, we consider a normalized version of the GCN forward pass, general to any message passing operator $\mathbf{P}$, 
\begin{equation}
    \X{\ell+1}
    =
    \frac{1}{\gamma}
    \frac{1}{\sqrt{D}}
    \mathbf{P}
    \X{\ell}
    \W{\ell+1}{},
    \label{eq:normalized-gcn-update}
\end{equation}
where $\gamma$ is a constant chosen to compensate for changes in scale induced by $\mathbf{P}$. 
To investigate the experimental impact of message passing normalization, we apply three message passing normalization strategies to models trained via Adam on the PascalVOC-SP dataset (Fig.~\ref{fig:normalization-sweep_pascal}) and, in Sec.~\ref{sec:mnist-message-passing-normalization}, on the MNIST-SP dataset (Fig.~\ref{fig:normalization-sweep_mnist}):
\begin{enumerate}
    \item Degree normalization, with $\mathbf{P}=\tilde{\A}$ and $\gamma=1$. 
    \item No normalization, with $\mathbf{P}=\A$ and $\gamma=1$. 
    \item $\gamma$ normalization, with $\mathbf{P}=\A$ and $\gamma$ chosen to be an empirical average over the per-graph scale shifts $||\mpop\mathbf{X}||_F/||\mathbf{X}||_F$ for each given dataset. For MNIST Superpixels we find $\gamma=17$ and for PascalVOC-SP $\gamma=7$. 
\end{enumerate}
The PascalVOC-SP sweeps in Fig.~\ref{fig:normalization-sweep_pascal} show that the degree normalized approach yields the sharpest learning rate transfer during early training (50 epochs). 
Later in training, the $\gamma$-normalized model performs best, and both normalized approaches exhibit learning rate transfer. 
The unnormalized model, while performant, does not yield sharp hyperparameter transfer in early or late training.
These experiments highlight that it is important to normalize the message passing operator, and that the exact strength of the normalization should be scanned as a hyperparameter ($\gamma$) in the vicinity of the average scale shifts induced by $\mathbf P$. 

\begin{figure}[h!]
\begin{centering}
    \includegraphics[width=0.65\textwidth]{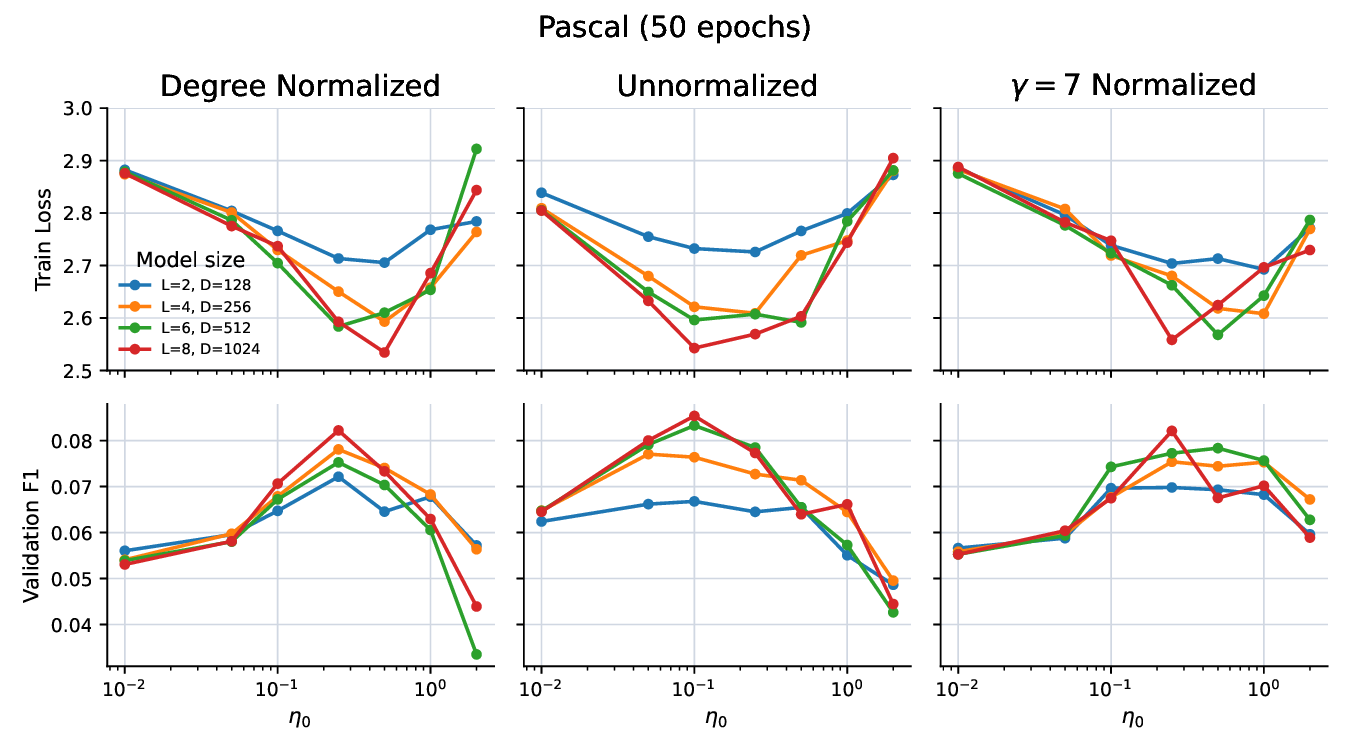}
    \includegraphics[width=0.65\textwidth]{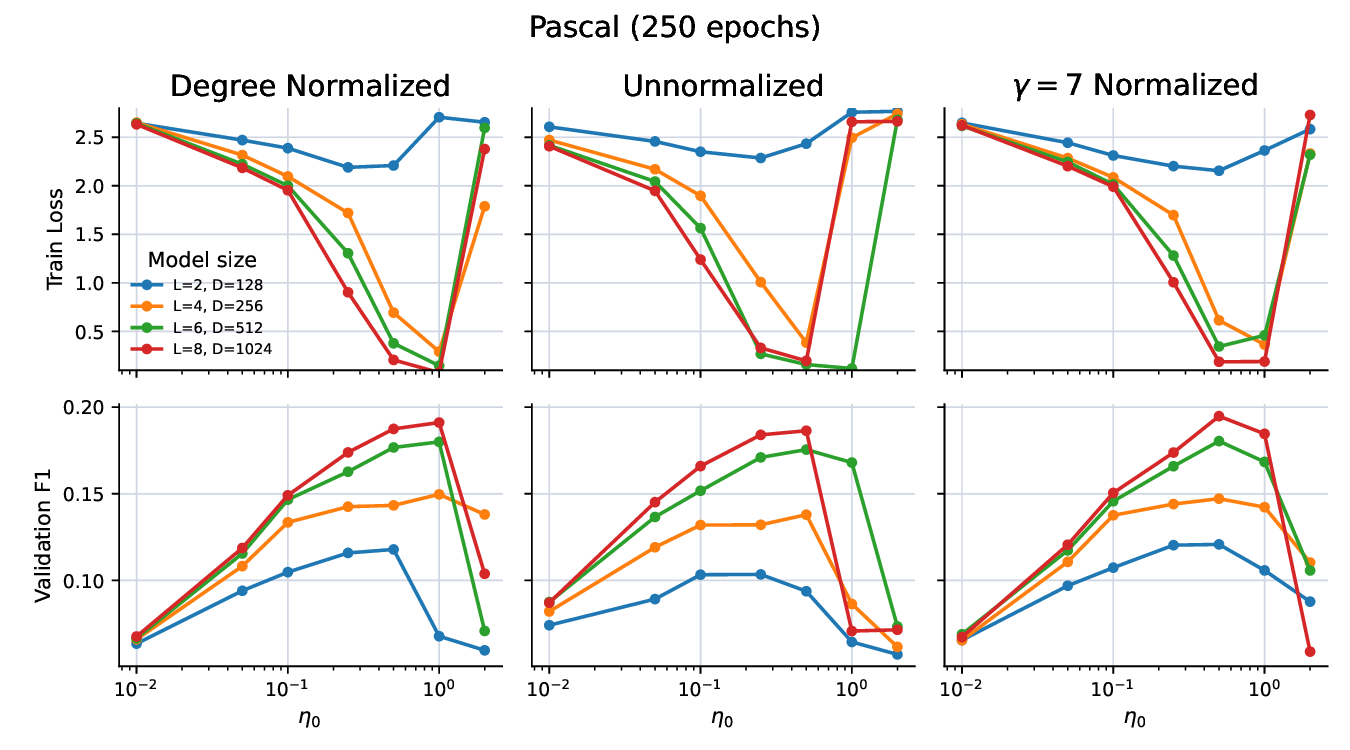}
    \caption{
    GNNs with degree normalization (left), no normalization (center), and $\gamma$ normalization (right) trained on the PascalVOC-SP dataset after 50 epochs (top) and 250 epochs (bottom). 
    }
    \label{fig:normalization-sweep_pascal}
\end{centering}
\end{figure}

\section{Discussion }

We have shown that the proposed GNN parameterization enables robust hyperparameter transfer and scalability across a range of graph learning tasks and optimizers. 
These results are significant because they support the broader claim that it is both possible and fruitful to train large (in particular, deep) GNNs.
Compared to language and vision models, relatively few scaling laws have been computed on graph data~\cite{liu2024towards}. 
Our parameterization enables inexpensive proxy tuning of large GNNs, making it possible to compute new scaling laws on graph data at substantially lower cost.
Though our scaling analysis focuses on simple, linear GNN blocks, our experiments validate that scaling predictions made in this simplified scenario persist into multi-epoch training of nonlinear GNNs. 

Graph learning introduces several numerical challenges. 
Graph inputs carry an explicit node dimension, and a downstream task may preserve or pool over it. 
Furthermore, input features may exhibit nontrivial structure, including sparsity, binarity, or strong correlations. 
These effects can change the scale of the features and their movements in the first layer, in some cases necessitating the use of a first-layer learning rate correction factor.
Though in Sec.~\ref{sec:scaling_analysis} we find different forms for this correction factor depending on the training scenario, it may be simplest to treat the correction factor as a hyperparameter $C$ when the quantity in Eqn.~\ref{eq:first-layer-correction} deviates significantly from $1$. 
An additional study in Sec.~\ref{sec:additional-fl-corr-study} explores $C$ and finds that for citation networks with sparse bag-of-words features, $C\gg 1$ leads to faster convergence than $C\approx 1$. 

In addition, message passing introduces graph-dependent linear operators whose normalization affects the scale of network activations and their updates. 
Across all our experiments, we find that symmetrically degree-normalized message passing operators yield sharp hyperparameter transfer. 
We also study $\gamma$ normalization (Eqn.~\ref{eq:normalized-gcn-update}), where an unnormalized message passing operator is used and a hyperparameter $\gamma$ is introduced to adjust the scale of the message passing update. 
Our results show that $\gamma$ normalization is also an effective strategy for controlling the scale of intermediate activations and their updates, leading to strong performance and hyperparameter transfer; in practice, scanning $\gamma$ in the vicinity of the scale shift induced by $\mathbf P$ (see Eqn.~\ref{eq:gamma-def}) may yield additional performance gains.

Several limitations to this work remain. 
While this analysis captures the dominant width and depth scaling we observe empirically, modern GNNs often include richer components like edge features~\cite{battaglia2016interactionnetworkslearningobjects}, attention~\cite{veličković2018graph,ladislav2022graphgps,shirzad2023exphormer}, Laplacian encodings~\cite{maskey2022generalized}, and deeper encoders and decoders. 
Extending this analysis to these additional architectural features is an important next step. 
Furthermore, our experiments are limited to moderate-scale, homogeneous benchmarks; transfer on large-scale, heterogeneous, temporal, and heterophilous graphs remains to be established.
Though our present results indicate that stable scaling is possible, future work should explore larger models and datasets, to make explicit comparisons with scaling laws in language and vision tasks. 



\bibliographystyle{unsrtnat}
\bibliography{reference}

\appendix

\section{Related Works}
\label{sec:review-of-literature}

Hyperparameter transfer was first developed in the Tensor Programs series \cite{yang2021zero,yang2022infwidth}. Specifically, Tensor Programs IV presents a ``maximal update'' NN parameterization ($\mu$P) that avoids lazy training (the neural tangent kernel (NTK) limit) in the infinite-width limit.  The key idea is to select width-dependent learning rates in each layer so that the updates $\Delta\mathbf{x}^{(\ell)}\sim O(1)$ as width grows. In Tensor Programs V, the authors present how $\mu$P allows for ``zero-shot hyperparameter transfer,'' whereby hyperparameter sweeps on small models produce scale-independent optima, across width~\cite{yang2021zero}. The width-wise $\mu$P scalings do not guarantee depth-wise transfer, which is generally achieved through the use of residual connections of strength $1/L^{\alpha_L}$. The literature has shown that while $\alpha_L\in [\tfrac{1}{2},1]$ leads to stable signal propagation, only $\alpha_L=1$ leads to full feature learning (see below).  In Tensor Programs VI~\cite{yang2024infdepth}, the authors use residual connections with $\alpha_L=\frac{1}{2}$ to establish joint width-depth hyperparameter transfer. 

A parallel body of work uses dynamical mean field theory (DMFT) to study hyperparameter transfer by taking infinite-width and/or infinite-depth limits of the network dynamics. These papers have broadly shown that $\mu$P scalings arise naturally from seeking non-degenerate mean-field training dynamics in the limit of infinite model scale. In Ref.~\cite{bordelon2023dmft}, the authors apply DMFT to feedforward NNs, ultimately recovering the $\mu$P width scalings. The joint infinite-width and infinite-depth scaling is explored in Ref.~\cite{bordelon2024depthwise}, which uses residual branches with $\alpha_L=\frac{1}{2}$, similarly recovering the $\mu$P prescription. Ref.~\cite{bordelon2024transformer} applies DMFT to transformers for general $\alpha_L$ and, additionally, $\alpha_A$, which controls the strength of the pre-attention scaling $\mathbf{k}\cdot\mathbf{q}/N^{\alpha_A}$.

The CompleteP prescription in Ref.~\cite{dey2025dontlazycompletepenables} builds on this prior DMFT work and presents a refined set of desiderata for hyperparameter transfer: 
\begin{enumerate}
    \item \textbf{Feature Stability}: Features $\mathbf{x}^{(\ell)}$ should remain stable and order-one with respect to model size.  
    \item \textbf{Update Stability}: Feature updates $\Delta\mathbf{x}^{(\ell)}$ should remain stable and nontrivial with respect to model size. 
    \item \textbf{Complete Learning}: Updates to features should not converge to their linearizations around any given parameter. In practice, this is satisfied by using $\alpha_L=1$ in each residual branch. 
\end{enumerate}
The first two desiderata follow from the previous works, whereas the third is a new prescription that prevents \textit{lazy} learning, in which some layers only learn features close to their linearization. Ref.~\cite{mlodozeniec2025completedhyperparametertransfermodules} presents a follow-up, dubbed Complete$^{(d)}$, which explores hyperparameter transfer across non-global hyperparameters belonging to the modules in each residual branch, as well as batch size and training duration. 

Graph neural networks are most often based on message passing. Simple architectures like the GCN~\cite{kipf2017semisupervised} are first-order spectral filters operating on graphs without edge features. Other variants like the GAT~\cite{veličković2018graph} compute attention scores in the neighborhood of each node. 
Transformers~\cite{vaswani2017attention} can be viewed as GNNs on fully-connected graphs, wherein messages computed from key-query attention-weighted value scores are propagated to each input token. Early variants of graph transformers embedded the key concept behind GAT, namely computing attention over the edges of an input graph, into this transformer architecture, optionally including edge features~\cite{dwivedi2021generalization}. GT architectures like Graphormer~\cite{ying2021transformersperformbad}, GraphGPS~\cite{ladislav2022graphgps}, and Exphormer~\cite{shirzad2023exphormer} expand on these concepts by including simple message passing mechanisms and positional encodings specific to the graph domain. Sparsity is a key motivation for graph transformers; while Exphormer attempts to expand input graphs, attempts have also been made to sparsify them~\cite{shirzad2024even}. 

Training deep GNNs is difficult because of numerical issues like oversmoothing~\cite{li2018deeper, oono2020graph} and oversquashing~\cite{alon2021on}. Several attempts have been made to produce GNNs trainable to large depth, including GCNII~\cite{chen2020simple}, DAGNN~\cite{meng2020towardsdeepergnns}, and Ref.~\cite{dezoort2023principles}. As a result, neural scaling laws on GNNs are underexplored. Early attempts as in Ref.~\cite{liu2024towards} report key GNN-specific observations like scaling law collapse (due to overfitting) and data volume corresponding to number of nodes or edges, not number of graphs. 

Hyperparameter transfer has not been as broadly applied to GNNs. 
Ref.~\cite{sypetkowski2024scalbilitygnns} presents a scaling analysis of GNNs on molecular data, using $\mu$P scalings for zero-shot scaling with respect to width, while facing numerical issues with their depth-$\mu$P scalings. 
It remains an open question how to apply hyperparameter transfer principles to graph transformers, whose architectures vary significantly.

\section{Benchmark Datasets}
\label{sec:benchmark-datasets}

We experiment on a collection of graph benchmarks spanning node classification, graph classification, and graph regression.

\paragraph{MNIST Superpixels.}
The MNIST-SP~\cite{lecun-mnisthandwrittendigit-2010} Superpixels dataset contains handwritten digit images embedded as graphs with superpixel nodes. 
It includes one node feature, the average  greyscale intensity in the superpixel. 
The task is to predict the digit represented by the superpixel graph. Full dataset details are given in Table~\ref{tab:graph-dataset-summary}.

\paragraph{PascalVOC-SP}
PascalVOC-SP~\cite{dwivedi2022lrgb} is a node classification dataset built from superpixels derived from the original Pascal VOC (Visual Object Class) image dataset. 
Each graph corresponds to an image, whose superpixels are common objects connected by edges encoding spatial adjacency. 
The task is to predict semantic object categories assigned to each superpixel node. Full dataset details are given in Table~\ref{tab:graph-dataset-summary}.

\paragraph{QM9 Dipole Moment Regression.}
QM9~\cite{ramakrishnan2014qm9} is a molecular property prediction benchmark consisting of small organic
molecules.
Molecules are embedded as graphs by assigning atoms to nodes and bonds to edges. 
Though the dataset contains multiple regression targets, we focus on the molecular dipole moment
\(\mu\). Full dataset details are given in Table~\ref{tab:graph-dataset-summary}. 

\paragraph{Citation Networks.}
Cora, CiteSeer, and PubMed~\cite{yang2016planetoid} are single-graph citation networks with node-labels and train/val/test masks; for this reason, the learning task is referred to as  semi-supervised node-classification. 
Edges are citations between papers embedded as nodes. 
Node features are sparse
bag-of-words vectors, and labels correspond to document topics.
Though the model makes predictions on every graph node via message passing, only those indicated by the nodes selected by the train mask are used during backpropagation.
Full dataset details are given in Table~\ref{tab:citation-dataset-summary}. 

\begin{table}[ht!]
    \centering
    \caption{
        Empirical benchmarks used throughout this paper; note that train/val/test
        splits are not unique across the literature and the table reports the splits used
        in our experiments. 
    }
    \renewcommand{\arraystretch}{1.15}
    \setlength{\tabcolsep}{5pt}
    \begin{tabular}{lcccccccl}
        \toprule
        \textbf{Dataset}
        & \textbf{Task}
        & \textbf{Train}
        & \textbf{Val}
        & \textbf{Test}
        & \textbf{Nodes}
        & \textbf{Edges}
        & \textbf{Feats}
        \\
        \midrule
        MNIST-SP
        & Graph Cls.
        & 54k
        & 6k
        & 10k
        & 75
        & 1393
        & 1 
        \\
        PascalVOC-SP
        & Node Cls.
        & 8{,}498
        & 1{,}428
        & 1{,}429
        & 479
        & 2{,}710
        & 14 
        \\
        QM9-\(\mu\)
        & Graph Reg.
        & $\approx$105k
        & $\approx$13k
        & $\approx$ 13k
        & 18
        & 37
        & 11 
        \\
        \bottomrule
        \vspace{0.5mm}
    \end{tabular}
    \label{tab:graph-dataset-summary}
\end{table}

\begin{table}[ht!]
    \centering
    \caption{
        The Planetoid semi-supervised citation-network benchmarks~\cite{yang2016planetoid}. Each dataset consists of a
        single citation graph with fixed train, validation, and test
        node masks. We use the canonical full split on all. 
    }
    \renewcommand{\arraystretch}{1.15}
    \setlength{\tabcolsep}{6pt}
    \begin{tabular}{lccccccccc}
        \toprule
        \textbf{Dataset}
        & \textbf{Nodes}
        & \textbf{Edges}
        & \textbf{Node Feats}
        & \textbf{Classes}
        \\
        \midrule
        Cora
        & 2{,}708
        & 10{,}556
        & 1{,}433 
        & 7
        \\
        CiteSeer
        & 3{,}327
        & 9{,}104
        & 3{,}703
        & 6
        \\
        PubMed
        & 19{,}717
        & 88,648
        & 500 
        & 3
        \\
        \bottomrule
        \vspace{0.5mm}
    \end{tabular}
    \label{tab:citation-dataset-summary}
\end{table}

\section{Run Specifications}
\label{sec:run-specifications}
Below we provide the specific run conditions used to produce each figure.
\begin{itemize}
    \item \textbf{Figure 1:} 
    GNNs with our transfer parameterization (top row) and with the standard parameterization (bottom row), both with symmetrically degree-normalized message passing
    operators ($\mpop=\tilde{\A}$), are trained via Adam
    ($\epsilon=10^{-14}$) on MNIST-SP graph classification (left),
    PascalVOC-SP node classification (center), and QM9 dipole moment regression
    (right). 
    A complete description of the standard parameterization is available in Sec.~\ref{sec:standard-param-baseline} below. 
    A batch size of 256 is used for all points; MNIST-SP is trained
    for 400 epochs, PascalVOC-SP is trained for 800 epochs, and QM9 is trained
    for 500 epochs. The best validation metric attained during training is reported. 
    \item \textbf{Figure 2:} Transfer GNNs with symmetrically degree-normalized message passing
    operators ($\mpop=\tilde{\A}$), trained with SGD (green), Adam (blue),
    and AdamW (orange). For each optimizer, the learning rate is fixed to
    the optimum identified in an $\eta_0$ transfer scan. MNIST-SP (left),
    PascalVOC-SP (center), and QM9 (right) are trained for 400, 800, and 500
    epochs, respectively, using a batch size of 256. The top row reports
    the best training loss attained during training, while the bottom row
    reports the corresponding validation metric.
    \item \textbf{Figure 3:} 
    Transfer GNNs with symmetrically degree-normalized adjacency matrices ($\mathbf{P}=\tilde{\A}$) are trained via SGD on MNIST-SP (left), PascalVOC-SP (center), and QM9 (right). 
    A batch size of 256 is used at all points; MNIST-SP is trained for 400 epochs, PascalVOC-SP is trained for 800 epochs, and QM9 is trained for 400 epochs. 
    The best train loss (validation metric) attained during training is reported in the top (bottom) row.
    \item \textbf{Figure 4:}
    Transfer GNNs with symmetrically degree-normalized adjacency matrices ($\mathbf{P}=\tilde{\mathbf A}$) are trained with (solid) and without (dashed) first-layer learning rate corrections on single-graph semi-supervised node classification datasets at optimal $\eta_0$ values determined in an auxiliary learning rate scan.
    \item \textbf{Figure 5:} 
    Transfer GNNs with symmetrically degree-normalized adjacency matrices ($\mathbf{P} = \tilde{\mathbf A}$) are trained via AdamW at various $\tau_\mathrm{epoch}$ working points, as a proxy for weight decay settings at fixed dataset size, batch size (256), and base learning rate. 
    All models are trained for 400 epochs. 
    \item \textbf{Figure 6:} 
    Transfer GNNs with symmetrically degree-normalized adjacency matrices ($\mathbf{P}=\tilde{\A}$) are trained via AdamW ($\epsilon=10^{-14}$) on MNIST-SP (left), PascalVOC-SP (center), and QM9 (right). 
    A batch size of 256 is used at all points; MNIST-SP ($\tepoch=150$) is trained for 400 epochs, PascalVOC-SP ($\tepoch=500$) is trained for 800 epochs, and QM9 ($\tepoch=1500$) is trained for 500 epochs. 
    The best train loss (validation metric) attained during training is reported in the top (bottom) row. 
    \item \textbf{Figure 7:} 
    Transfer GNNs with degree normalization (left), no normalization (center), and $\gamma$ normalization (right) are trained via Adam ($\epsilon=10^{-14}$) on the PascalVOC-SP dataset.
    Early training performance (after 50 epochs) is shown in the upper panel and late training performance (after 250 epochs) in the lower panel. 
    The best training and validation scores achieved are reported in each of the upper and lower subpanels respectively.
    \item \textbf{Figure 8:} 
    Transfer GNNs with symmetrically degree-normalized adjacency matrices ($\mathbf P=\tilde{\A}$) are trained via GD on CiteSeer and PubMed at various $(C, \eta_0)$ working points, where $C$ is a first-layer correction factor. The number of epochs required to reach $70\%$ validation accuracy on CiteSeer and $85\%$ accuracy on PubMed are reported at each working point. 
    \item \textbf{Figure 9:} 
    Transfer GNNs with degree normalization (left), no normalization (center), and $\gamma$ normalization (right) are trained via Adam ($\epsilon=10^{-14}$) on the MNIST-SP dataset.
    Early training performance (after 10 epochs) is shown in the upper panel and late training performance (after 50 epochs) in the lower panel. 
    The best training and validation scores achieved are reported in each of the upper and lower subpanels respectively.
\end{itemize}

\subsection{Standard parameterization baseline}
\label{sec:standard-param-baseline}
As a benchmark of our proposed parameterization, we compare to a \textit{standard parameterization} baseline that differs in how initialization scales, residual-branch scales, and optimizer
hyperparameters depend on the width $D$ and depth $L$. 
The baseline retains the same overall model structure, including normalized message-passing via $\tilde{\mathbf A}$. 
The depth scaling, multiplying the residual contributions, is modified to be $L^{-1/2}$ rather than $L^{-1}$. 
Every trainable linear map is initialized with LeCun fan-in variance and is applied with no additional forward
multiplier. 
As in the transfer parameterization, no biases are used. 
In particular, the baseline uses no $\sigma_0$ or $\sigma_{L+1}$ initialization scales and no width-dependent encoder or readout multipliers. 
The baseline consumes the configured optimizer hyperparameters verbatim: every parameter group uses the same learning rate $\eta = \eta_0$ independent of $D$ and $L$.
Table~\ref{tab:our_vs_standard_param} summarizes these differences.

\begin{table}[t]
  \centering
  \caption{Width and depth dependence of our transfer parameterization and the standard parameterization when trained via Adam. For our transfer parameterization, choosing $\sigma_0=\sigma_{L+1}=1/\sqrt{D}$ allows for the use of a global learning rate $\eta_0/\sqrt{D}$. All architectural choices are held fixed, including the use of a normalized message passing operator.}
  \begin{tabular}{lll}
    \toprule
    \textbf{Parameterization} & Ours                         & Standard \\
    \midrule
    Encoder init.\ variance   & $1/D$                        & $1/n_0$ \\
    Encoder fwd.\ scale       & $\sqrt{D/n_0}$               & $1$ \\
    Residual init.\ variance    & $1$                        & $1/\text{fan-in}$ \\
    Residual fwd.\ scale        & $1/\sqrt{\text{fan-in}}$   & $1$ \\
    Residual multiplier       & $1/L$                        & $1/\sqrt{L}$ \\
    Decoder init.\ variance   & $1/D$                        & $1/D$ \\
    Decoder fwd.\ scale       & $1/\sqrt{D}$                 & $1$ \\
    Adam learning rate        & $\eta_0/\sqrt{D}$            & $\eta_0$ \\
    \bottomrule
  \end{tabular}
    \label{tab:our_vs_standard_param}
\end{table}

\section{Scaling Analysis}
\label{sec:scaling_analysis}

The following calculations show how the scalings in Eqns.~\ref{eq:base-model-encoder}-\ref{eq:base-model-decoder} and Table~\ref{tab:learning_rate_settings} arise in GNNs.   
Broadly, our goal is to satisfy the desiderata for hyperparameter transfer laid out in Ref.~\cite{dey2025dontlazycompletepenables}, 
\begin{enumerate}
    \item \textbf{Feature Stability}: Features $\x{\ell}{u;i}$ should remain stable and order-one with respect to model size, i.e. $\x{\ell}{u}\sim\Theta_{D,L}(1)$ in each layer.  
    \item \textbf{Update Stability}: Feature updates $\Delta\x{\ell}{u;i}$ should remain stable and nontrivial with respect to model size. In the encoder and decoder layers, we require that $\Delta \x{0}{u;i}$ and $\Delta \x{L+1}{u;i}$ remain $\Theta_{D,L}(1)$. 
    In the residual stream, we require that $\Delta \x{\ell}{u;i}\sim\Theta_{D,L}(1/L)$ (for $1\leq\ell\leq L$), so that the total contribution of the residual stream remains $\Theta_{D,L}(1)$. 
    \item \textbf{Complete Learning}: Updates to features should not converge to their linearizations around any given parameter. In practice, this is satisfied by using $1/L$ depth scalings in each residual connection; see Ref.~\cite{dey2025dontlazycompletepenables}. 
\end{enumerate}

\subsection{Minimal Transfer GNN}
\label{sec:minimal-transfer-gnn}

To illustrate how the hyperparameter desiderata are satisfied by our base model in Eqns.~\ref{eq:base-model-encoder}-\ref{eq:base-model-decoder}, consider a \textit{minimal transfer GNN} in which no nonlinearities are applied. 
Throughout assume that the input node features are $\ell_2$ normalized to $n_0$ as 
\begin{equation}
    ||\mathbf{x}_u||_2^2=n_0 
    \qquad
    \rightarrow
    \qquad 
    \frac{1}{Nn_0}||\mathbf{X}||_F^2 = 1,
\label{eq:input-normalization}
\end{equation}
so that the RMS input entry is order-one. 
The first-layer, corresponding to Eqn.~\ref{eq:base-model-encoder}, is a simple linear encoder 
\begin{equation}
    \X{0} = 
    \frac{1}{\sigma_0\sqrt{n_0}}
    \mathbf{X}\W{0}{}
    \label{eq:simple-encoder}
\end{equation}
with $\W{0}{ij}\sim\mathcal{N}(0,\sigma_0^2)$. 
The residual stream, corresponding to Eqns.~\ref{eq:base-model-mpnn} and~\ref{eq:base-model-mlp}, will use a GCN as the MPNN module,
\begin{equation}
    \Xt{\ell+1} = 
    \mathbf{X}^{(\ell)} + 
    \frac{1}{L}
    \bigg(
        \frac{1}{\sqrt{D}}
        \frac{1}{\gamma_\ell}
        \mpop\X{\ell}\Wt{\ell+1}
    \bigg)
    \label{eq:simple-mpnn}
\end{equation}
with a generalized message passing operator $\mpop\in\R^{N\times N}$, weights $\Wt{\ell+1}_{ij}\sim\mathcal{N}(0,1)$, and factors of
\begin{equation}
    \gamma_\ell^2 := 
    \Exp{\norm{\mpop\X{\ell}}_F^2}
    / \ 
    \Exp{\norm{\X{\ell}}_F^2}
    \label{eq:gamma_ell}
\end{equation}
introduced to correct for changes in scale induced by $\mpop$. 
Note that for practical purposes, the $\gamma_\ell$ factors can be treated as a single hyperparameter $\gamma$ (see Eqn.~\ref{eq:normalized-gcn-update}). 
Our empirical results show that using a normalized message passing operator like $\tilde{\mathbf A}$ with $\gamma=1$ yields excellent learning rate transfer. 
Next, we take an MLP with no nonlinearities and one hidden layer of width $4D$,
\begin{equation}
    \X{\ell+1} 
    = 
    \Xt{\ell+1}
    + 
    \frac{1}{L}
    \bigg(
        \frac{1}{\sqrt{4D}}
        \bigg(
            \frac{1}{\sqrt{D}}
            \Xt{\ell+1}
            \W{\ell+1}{1}
        \bigg)
        \W{\ell+1}{2}
    \bigg)
    \label{eq:simple-mlp}
\end{equation}
with $(\W{\ell+1}{1})_{ij}, (\W{\ell+1}{2})_{ij}\sim\mathcal{N}(0,1)$.
Finally, we take 1-D outputs using a simple linear decoder, either per-graph
\begin{equation}
    z = 
    \frac{1}{D\sigma_{L+1}}\frac{1}{N}
    \one^T\mathbf{X}^{(L)}\W{L+1}{}
    \in\R 
    \label{eq:simple_decoder-graph}
\end{equation}
or per-node
\begin{equation}
    \mathbf Z = 
    \frac{1}{D\sigma_{L+1}}\mathbf{X}^{(L)}\W{L+1}{}
    \in\R^{N}
    \label{eq:simple-decoder-node}
\end{equation}
with $\W{L+1}{ij}\sim\mathcal{N}(0,\sigma_{L+1}^2)$.

\subsection{Feature Stability}
\label{sec:scaling-analysis_feature-stability}

Desiderata 1 demands numerical stability in the forward pass, in the sense that the model's activations remain order-one with respect to model size. 

\begin{proposition}
\label{th:feature_stability}
    \textbf{(Feature Stability)}
    Fix an input graph $\mathcal{G} = (\mathcal{V}, \mathcal{E})$ equipped with an adjacency matrix $\mathbf{A} \in \mathbb{R}^{N\times N}$ and node features $\mathbf{X} \in \mathbb{R}^{N \times n_0}$ normalized for every $u \in \mathcal{V}$ as $\|\mathbf{x}_u\|^2_2 = n_0$. Then, for the minimal transfer GNN,
    \begin{itemize}
        \item The encoded features $\x{0}{v;i}$ satisfy $\Theta_{D,L}(1)$ scaling.
        \item The residual features $\x{1\leq \ell\leq L}{v;i}$ satisfy $\Theta_{D,L}(1)$ scaling.
        \item The scalar output $x^{(L+1)}$ satisfies $z = \Theta\left( M_L/D\right)$ scaling, where
        \begin{equation}
            M_\ell := 
            \left(
                \Exp{ 
                    \left\|
                        \frac{1}{N}
                        \mathbf{1}^T
                        \X{\ell}{}
                    \right\|_2^2
                }
            \right)^{1/2}.
            \label{eq:m_ell}
        \end{equation}
    \end{itemize}
\end{proposition}
\begin{proof} We will address the encoder, residual, and decoder layers in sequence. 

\paragraph{Encoder Layer} The outputs of the encoder layer have size
    \[
        \Exp{\norm{\X{0}}_F^2} 
        = 
        \frac{1}{\sigma_0^2 n_0}
        \Exp{\norm{\mathbf{X}\W{0}{}}_F^2}
        =
        \frac{D}{n_0}
        \Exp{\norm{\mathbf{X}}_F^2} =
        ND,
    \]
where we used the initialization scale $\W{0}{ij}\sim\mathcal{N}(0,\sigma_0^2)$ and the normalization $||\mathbf{x}_u||^2_2=n_0$. 
Because $\X{0}$ contains $ND$ activations, each activation satisfies $||\x{0}{v;i}||\sim\Theta_{D,L}(1)$. 

\paragraph{Residual Layers} The outputs of each residual layer have size, 
    \begin{align*}
        \Exp{\norm{\X{\ell+1}}_F^2}
        &= 
        \Exp{\norm{
            \Xt{\ell+1}
            \bigg(
                \ID_D 
                + 
                \frac{1}{L}
                \frac{1}{\sqrt{D}}
                \W{\ell+1}{1}
                \frac{1}{\sqrt{4D}}
                \W{\ell+1}{2}
            \bigg)
        }_F^2}
        \\
        &=
        \Exp{\norm{\Xt{\ell+1}}_F^2}
        \bigg(
            1 + \frac{1}{L^2}
        \bigg)
    \end{align*}
    where the cross terms containing only one copy of each weight matrix vanish in expectation.
    Then, because
    \begin{align*}
        \Exp{\norm{\Xt{\ell+1}}_F^2}
        &= 
        \Exp{\norm{\X{\ell}}_F^2} 
        + 
        \frac{1}{L^2}
        \frac{1}{\gamma_\ell^2}
        \Exp{\norm{
            \mpop 
            \X{\ell}
            \frac{1}{\sqrt{D}}
            \Wt{\ell+1}
        }_F^2}
        \\
        &= 
        \Exp{\norm{\X{\ell}}^2_F} +
        \frac{1}{L^2}
        \frac{1}{\gamma_\ell^2}
        \Exp{\norm{\mpop\X{\ell}}_F^2}
        \\
        &= 
        \Exp{\norm{\X{\ell}}^2_F}
        \bigg(
            1 + \frac{1}{L^2}
        \bigg),
    \end{align*}
    the scale of features in the residual stream satisfy a simple recursion in $\ell$:
    \begin{align*}
        \Exp{\norm{\X{\ell+1}}_F^2}
        &= 
        \Exp{\norm{\X{\ell}}^2_F}
        \bigg(
            1 + \frac{1}{L^2}
        \bigg)^2
        \\
        &\sim 
        ND + \mathcal{O}(L^{-2}),
    \end{align*}
    which guarantees RMS stability and practically keeps a typical activation $||\x{\ell}{v;i}||\sim\Theta_{D,L}(1)$ for $1\leq\ell\leq L$.

    \paragraph{Decoder Layer}
    Assuming, without loss of generality, a graph-regression task with one target, the outputs have a suppressed scale at init,
    \begin{align*}
        \Exp{(x^{(L+1)})^2} = 
        \frac{1}{\sigma_{L+1}^2D^2}
        \Exp{
            \bigg(
                \frac{1}{N}\one^T \X{L}\W{L+1}{}
            \bigg)^2
        }
        = 
        \frac{1}{D^2}
        \Exp{\norm{\frac{1}{N}\one^T \X{L}}_2^2}
        =
        \frac{M_L^2}{D^2},
    \end{align*}
    so that $x^{(L+1)}\sim\Theta(M_L/D)$.
\end{proof}

\subsection{Update Stability}
\label{sec:scaling-analysis-update-stability}

We now investigate the scale of updates to the network's features.
In the following calculations we'll consider making a prediction on  graph $\mathcal{G}_\alpha=(\vset_\alpha, \eset_\alpha)$, backpropagating the results to update the model's weights, and then making a new prediction on a second graph $\mathcal{G}_\beta=(\vset_\beta, \eset_\beta)$.
We'll denote the number of nodes in each graph $N_\alpha:=|\vset_\alpha|$ and $N_\beta:=|\vset_\beta|$ respectively. Similarly, their truth targets are $y_\alpha$ and $y_\beta$, their node features are $\mathbf{X}_\alpha\in\R^{N_\alpha \times n_0}$ and $\mathbf{X}_\beta\in\R^{N_\beta \times n_0}$ (with the features belonging to each node $\ell_2$ normalized to $n_0$ as in Prop.~\ref{th:feature_stability}), and their adjacency matrices are $\A_\alpha\in\mathbb{R}^{N_\alpha\times N_\alpha}$ and $\A_\beta\in\mathbb{R}^{N_\beta\times N_\beta}$. 

\begin{proposition}
\label{th:update-stability-gd-lr-graph-regression}
    (\textbf{GD Learning Rates})
    After one step of training a minimal transfer GNN on $\mathcal{G}_\alpha$ via GD with a quadratic loss, network features in the encoder, residual, and decoder layers receive $\Theta(1)$ updates after processing a second graph $\mathcal{G}_\beta$ if
    \begin{itemize}
        \item The learning rate in the first-layer is 
        \begin{equation}
            \etasgd^{(0)} = \eta_0D\sigma_0^2 \times\Cab
            \ \ 
            \text{with}
            \ \ 
            \Cab:=\frac{n_0\sqrt{N_\beta}}{\Mab}
            \label{eq:lr_sgd_first}
        \end{equation}
        where 
        \begin{equation}
            \Mab := \norm{
                \frac{1}{N_\alpha}
                \one_{N_\alpha}^T
                \mathbf{X}_\alpha
                (\mathbf{X}_\beta)^T
            }_2
            \label{eq:Mab}
        \end{equation}
        measures the node feature alignment of $\mathcal G_\alpha$ and $\mathcal G_\beta$. 
        \item The learning rate in the residual layers is 
        \begin{equation}
            \etasgd^{(\ell)} = \eta_0DL.
            \label{eq:eta_sgd_resid}
        \end{equation}
        \item The learning rate in the last layer is 
        \begin{equation}
            \etasgd^{(L+1)} = \eta_0 D\sigma_{L+1}^2. 
        \end{equation}
    \end{itemize}
    Though different learning rates are required in each of the encoder, residual, and decoder layers, one may use a global effective learning rate $\etasgd=\eta_0DL$ with the choices $\sigma_0=\sqrt{L}$ and $\sigma_{L+1}=\sqrt{L}$, assuming that $\Cab$ is close enough to 1. If $\Cab$ varies significantly from 1, a separate learning rate should be used in the encoder layer, whose correction factor can be treated as a hyperparameter scanned in the neighborhood of $\Cab$. 
\end{proposition}

\begin{proof}
For a 1D graph regression task with predictions $z_\alpha$, the MSE loss and corresponding residual are 
\[
    \mathcal{L}_\alpha = 
    \frac{1}{2}
    \big(
        y_\alpha - z_\alpha
    \big)^2
    \ \  \ \ 
    \text{and}
    \ \  \ \ 
    \Delta_\alpha :=
    -\nabla_{z_\alpha} \mathcal{L}_\alpha =
    y_\alpha - z_\alpha. 
\] 
\paragraph{First Layer}
Under GD, the scale of the updates to the first-layer features is
\[
    \Exp{\norm{\Delta\X{0}_\beta}_F^2} 
    = 
    \frac{1}{n_0\sigma_0^2}
    \Exp{
        \norm{
            \mathbf{X}_\beta\Delta\W{0}{}
        }_F^2
    }
    = 
    \frac{(\Delta_\alpha \etasgd^{(0)})^2}{n_0\sigma_0^2}
    \Exp{
        \norm{
            \mathbf{X}_\beta 
            \nabla_{\W{0}{}} z_\alpha 
        }_F^2
    }.
\] 
The gradient with respect to the first-layer weights is
\begin{align}
    \nabla_{\W{0}{}} z_\alpha 
    &= 
    \nabla_{\W{0}{}}
    \bigg(
        \frac{1}{\sigma_{L+1} DN_\alpha}
        \one_{N_\alpha}^T
        \bigg(
            \big(
                \X{L-1}_\alpha
                + 
                \frac{1}{L}
                \frac{1}{\gamma_{L-1}}
                \mathbf P_\alpha
                \X{L-1}_\alpha
                \frac{1}{\sqrt{D}}
                \Wt{L}
            \big)
            \big(
                \ID_D
                + 
                \frac{1}{L\sqrt{D}}
                \W{L}{1}
                \frac{1}{\sqrt{4D}}
                \W{L}{2}
            \big)
        \bigg)
        \W{L+1}{}
    \bigg)
    \nonumber
    \\
    &\approx
    \nabla_{\W{0}{}}
    \bigg(
        \frac{1}{\sigma_{L+1}DN_\alpha}
        \one_{N_\alpha}^T
        \X{0}_\alpha
        \bigg(
            \prod_{\ell=0}^{L-1}
            \big(
                \ID_D
                + 
                \frac{1}{L}
                \frac{1}{\sqrt{D}}
                \Wt{\ell+1}
            \big)
            \big(
                \ID_D
                + 
                \frac{1}{L\sqrt{D}}
                \W{\ell+1}{1}
                \frac{1}{\sqrt{4D}}
                \W{\ell+1}{2}
            \big)
        \bigg)
        \W{L+1}{}
    \bigg)
    \nonumber
    \\
     &:=
    \nabla_{\W{0}{}}
    \bigg(
        \frac{1}{\sigma_{L+1} DN_\alpha}
        \one_{N_\alpha}^T
        \bigg(
            \frac{1}{\sigma_0\sqrt{n_0}}
            \mathbf{X}_\alpha 
            \W{0}{}
        \bigg)
        \mathbf{R}^{(0)}
        \W{L+1}{}
    \bigg)
    \nonumber
    \\
    &=
    \frac{1}{\sigma_{L+1}\sigma_0 DN_\alpha\sqrt{n_0}}
    \nabla_{\W{0}{}}
    \text{Tr}
    \bigg(
        \one_{N_\alpha}^T
        \mathbf{X}_\alpha 
        \W{0}{}
        \mathbf{R}^{(0)}
        \W{L+1}{}
    \bigg)
    \nonumber
    \\
    &=
    \frac{1}{\sigma_{L+1}\sigma_0 DN_\alpha\sqrt{n_0}}
    \bigg(
        \mathbf{R}^{(0)}
        \W{L+1}{}
        \one_{N_\alpha}^T
        \mathbf{X}_\alpha
    \bigg)^T.
    \label{eq:first-layer-gradient}
\end{align}
where in the second line we approximate that the message passing operator $\mathbf P_\alpha$, when properly normalized by \(\gamma_\ell\) (or otherwise degree-normalized such that $\gamma_\ell \approx 1$), is treated as contributing order-one factors to the forward and backward feature norms. This approximation isolates the dependence of the update on width and depth; graph-pair-dependent effects are represented by the alignment factor \(\Mab\) introduced in Eqn.~\ref{eq:Mab}. 
In the third line we defined the residual contribution
\begin{equation}
    \mathbf{R}^{(\ell)}
    :=
    \prod_{\ell'=\ell}^{L-1}
    \big(
        \ID_D + 
        \frac{1}{L}
        \frac{1}{\sqrt{D}}
        \Wt{\ell'+1}
    \big)
    \big(
        \ID_D + 
        \frac{1}{L}
        \frac{1}{\sqrt{D}}
        \W{\ell'+1}{1}
        \frac{1}{\sqrt{4D}}
        \W{\ell'+1}{2}
    \big).
    \label{eq:residual_contribution}
\end{equation}
By design, $\mathbf R^{(\ell)}$ contributes an order-one factor to the scale of the gradient, in the sense that 
\begin{align}
    \Exp{(\mathbf{R}^{(\ell)})^T\mathbf{R}^{(\ell)}}
    &=
    \prod_{\ell'=\ell}^{L-1}
    \Exp{
        \ID_D + 
        \frac{1}{L^2D}
        (\Wt{\ell'+1})^T\Wt{\ell'+1} + 
        \frac{1}{4L^2D^2}
        (\W{\ell'+1}{1}
        \W{\ell'+1}{2})^T
        \W{\ell'+1}{1}
        \W{\ell'+1}{2}
        + \mathcal{O}(L^{-4})
    } 
    \nonumber
    \\
    &=
    \ID_D
    \times
    \big(
        1
        + 
        \frac{2}{L^2} 
    +  \mathcal{O}(L^{-4})
    \big)^{L-\ell},
    \label{eq:RRT}
\end{align}
where we have used the independence of the different weight types in each layer. 
Thus, the first-layer updates scale as
\begin{align*}
    \Exp{\norm{\Delta\X{0}_\beta}_F^2} 
    &\sim
    \frac{(\Delta_\alpha \etasgd^{(0)})^2}{n_0\sigma_0^2}
    \frac{1}{(\sigma_{L+1}\sigma_0 DN_\alpha)^2n_0}
    \Exp{
        \norm{
            \mathbf{X}_\beta 
            (\mathbf{X}_\alpha)^T
            \one_{N_\alpha}
            (\W{L+1}{})^T
            (\mathbf{R}^{(0)})^T
        }_F^2
    }
    \\
    &=
    \frac{(\Delta_\alpha \etasgd^{(0)})^2}{n_0^2\sigma_0^4D}
    \bigg(
        1 + 
        \frac{2}{L^2}
        +
        \mathcal{O}(L^{-4})
    \bigg)^L
    \norm{
        \frac{1}{N_\alpha}
        \one_{N_\alpha}^T
        \mathbf{X}_\alpha
        (\mathbf{X}_\beta)^T
    }_2^2
    \\
    &\sim
    \frac{(\etasgd^{(0)})^2}{n_0^2\sigma_0^4D}
    \Mab^2. 
\end{align*}
To keep this Frobenius norm $\sim N_\beta D$, the learning rate should be parameterized as 
\[
    \etasgd^{(0)} = \eta_0D\sigma_0^2
    \times
    \frac{n_0\sqrt{N_\beta}}{\Mab}. 
\]
Therefore, we may use $\sigma_{0}$ to modify the gradient scale in the first layer, which additionally depends on a correction factor related to $\Mab$, a quantity tracking the correlation of the inputs.

\paragraph{Residual Layers}
In the $(\ell+1)^\mathrm{th}$ layer we have
\begin{align*}
    \X{\ell+1}_\beta
    &=
    \bigg(
        \mathbf X^{(\ell)}_\beta +
        \frac{1}{L\sqrt{D}}\frac{1}{\gamma_\ell}
        \mathbf P_\beta
        \X{\ell}_\beta
        \Wt{\ell+1}
    \bigg)
    \bigg(
        \ID_D +
        \frac{1}{L\sqrt{D}}\W{\ell+1}{1}
        \frac{1}{\sqrt{4D}}\W{\ell+1}{2}
    \bigg)
    \\
    &\approx
    \X{\ell}_\beta
    \bigg(
        \ID_D +
        \frac{1}{L\sqrt{D}}\Wt{\ell+1}
    \bigg)
    \bigg(
        \ID_D +
        \frac{1}{L\sqrt{D}}\W{\ell+1}{1}
        \frac{1}{\sqrt{4D}}\W{\ell+1}{2}
    \bigg)
    \\
    &=
    \X{\ell}_\beta
    +
    \frac{1}{L\sqrt{D}}\X{\ell}_\beta\Wt{\ell+1}
    +
    \frac{1}{L\sqrt{D}}\X{\ell}_\beta\W{\ell+1}{1}
    \frac{1}{\sqrt{4D}}\W{\ell+1}{2}
    +
    \mathcal{O}(L^{-2})
\end{align*}
so that we must separately consider gradients with respect to $\W{\ell+1}{2}$, $\W{\ell+1}{1}$, and $\Wt{\ell+1}$. 
In the second line, we approximated that the message passing operator $\mathbf P_\beta$, when appropriately normalized by $\gamma_\ell$, will not affect the width or depth scaling of the network. 
Note that if we do not scale the forward pass by $\gamma_\ell^{-1}$, the second and third terms inherit different scales; the message passing term will have a size $\gamma_\ell$, shifting its contribution relative to that of the MLP weights.
Updates to $\X{\ell+1}_\beta$ take the form
\[
    \Delta \X{\ell+1}_\beta
    = 
    \Delta \X{\ell}_\beta
     + 
     \frac{1}{L\sqrt{D}}
     \Delta\bigg(
        \X{\ell}_\beta\Wt{\ell+1}
     \bigg)
     +
     \frac{1}{2LD}
     \Delta\bigg(
        \X{\ell}_\beta\W{\ell+1}{1}\W{\ell+1}{2}
    \bigg). 
\]
In the first-layer calculation above, we chose a learning rate such that $||\Delta \X{0}_\beta||_F \sim \sqrt{N_\beta D}$, ensuring every component is order-one. 
As the residual updates are computed recursively from this base case, we expect the first term above to scale like $||\Delta\X{\ell}_\beta||_F \sim \sqrt{N_\beta D}$.
The second and third terms are suppressed by a factor of $1/L$, and we therefore expect them to have Frobenius norms $\sqrt{N_\beta D} / L$. 
This fixes the learning rate in the residual layers; to illustrate, consider the scale of updates produced by the second term:
\begin{align*}
    \frac{1}{L^2D}
    \Exp{
    \norm{
        \Delta\bigg(
            \X{\ell}_\beta
            \Wt{\ell+1}
        \bigg)
    }_F^2}
    &= 
    \frac{1}{L^2D}
    \bigg(
        \Exp{
        \norm{
            \Delta \X{\ell}_\beta
            \Wt{\ell+1}
        }_F^2}
        +
        \Exp{
        \norm{
            \X{\ell}_\beta
            \Delta\Wt{\ell+1}
        }_F^2}
    \bigg)
    \\
    &\sim
    \frac{N_\beta D}{L^2} 
    + 
    \frac{(\etasgd^{(\ell+1)}\Delta_\alpha)^2}{L^2D}
    \Exp{
        \norm{
            \X{\ell}_\beta
            \nabla_{\Wt{\ell+1}} z_\alpha 
        }_F^2
    }
\end{align*}
where the gradient (approximating as usual that each scale shift induced by $\mathbf P_\alpha$ is compensated by its corresponding normalization by $\gamma_\ell$) scales like 
\begin{align}
    \nabla_{\Wt{\ell+1}} z_\alpha 
    &\approx
    \nabla_{\Wt{\ell+1}}
    \bigg(
        \frac{1}{\sigma_{L+1}DN_\alpha}
        \one_{N_\alpha}^T
        \bigg(
            \X{\ell}_\alpha
            \prod_{\ell'=\ell}^{L-1}
            \big(
                \ID_D
                + 
                \frac{1}{L}
                \frac{1}{\sqrt{D}}
                \Wt{\ell'+1}
            \big)
            \big(
                \ID_D
                + 
                \frac{1}{L\sqrt{D}}
                \W{\ell'+1}{1}
                \frac{1}{\sqrt{4D}}
                \W{\ell'+1}{2}
            \big)
        \bigg)
        \W{L+1}{}
    \bigg)
    \nonumber
    \\
    &=
    \nabla_{\Wt{\ell+1}}
    \bigg(
        \frac{1}{\sigma_{L+1}DN_\alpha}
        \one_{N_\alpha}^T
        \bigg(
            \X{\ell}_\alpha
            \big(
                \ID_D
                +
                \frac{1}{L}
                \frac{1}{\sqrt{D}}
                \Wt{\ell+1}
            \big)
            \big(
                \ID_D
                +
                \frac{1}{2LD}
                \W{\ell+1}{1}
                \W{\ell+1}{2}
            \big)
            \mathbf{R}^{(\ell+1)}
        \bigg)
        \W{L+1}{}
    \bigg)
    \nonumber
    \\
    &=
    \frac{1}{\sigma_{L+1}D^{3/2}N_\alpha L}
    \nabla_{\Wt{\ell+1}}
    \tr{
    \bigg(
        \one_{N_\alpha}^T
        \X{\ell}_\alpha
        \Wt{\ell+1}       
        \mathbf{R}^{(\ell+1)} 
        \W{L+1}{}
    \bigg)
    }
    + \mathcal{O}(L^{-2})
    \nonumber
    \\
    &=
    \frac{1}{\sigma_{L+1}D^{3/2}N_\alpha L}
    \bigg(
        \mathbf{R}^{(\ell+1)}
        \W{L+1}{}
        \one_{N_\alpha}^T
        \X{\ell}_\alpha
    \bigg)^T
    + \mathcal{O}(L^{-2}).
    \label{eq:grad_Wl}
\end{align}
With this, ignoring the gradient contribution suppressed by $L^{-2}$, we have 
\begin{align}
    \frac{(\etasgd^{(\ell+1)}\Delta_\alpha)^2}{L^2D}
    \Exp{
    \norm{
        \X{\ell}_\beta
        \nabla_{\Wt{\ell+1}} z_\alpha 
    }_F^2}
    &\approx
    \frac{(\etasgd^{(\ell+1)}\Delta_\alpha)^2}{L^2D}
    \frac{1}{\sigma_{L+1}^2D^3 N_\alpha^2 L^2}
    \Exp{
    \norm{
        \X{\ell}_\beta
        \bigg(
            \mathbf{R}^{(\ell+1)}
            \W{L+1}{}
            \one_{N_\alpha}^T
            \X{\ell}_\alpha
        \bigg)^T
    }_F^2
    }
    \nonumber
    \\
    &=
    \frac{
        (\etasgd^{(\ell+1)}\Delta_\alpha)^2
    }{L^4D^3}
    \bigg(
        1
        + 
        \frac{2}{L^2} 
    \bigg)^{L-\ell-1}
    \Exp{
    \norm{
        \frac{1}{N_\alpha}
        \X{\ell}_\beta
        \big(
            \X{\ell}_\alpha
        \big)^T
        \one_{N_\alpha}
    }_F^2
    }
    \nonumber
    \\
    &:\sim
    \frac{
        (\etasgd^{(\ell+1)})^2
    }{L^4D^3}
    (\Mab^{(\ell)})^2
    \nonumber
    \\
    &\sim
    \frac{
        (\etasgd^{(\ell+1)})^2 N_\beta
    }{L^4D}
    \label{eq:resid_grad_term}
\end{align}
where we have defined the correlation quantity $\Mab^{(\ell)}$ analogously to $\Mab$ (see Eqn.~\ref{eq:Mab}) as
\begin{equation}
    (\Mab^{(\ell)})^2 :=
    \Exp{
        \norm{
            \frac{1}{N_\alpha}
            \one_{N_\alpha}^T
            \X{\ell}_\alpha
            \big(
                \X{\ell}_\beta
            \big)^T
        }_2^2
        }.
    \label{eq:Mab_ell}
\end{equation}
and assumed order-one feature alignment so that $\Mab^{(\ell)}\sim D\sqrt{N_\beta}$, motivated by the recursive feature scale assumptions $||\X{\ell}_\alpha|| \sim \sqrt{N_\alpha D}$ and $||\X{\ell}_\beta|| \sim \sqrt{N_\beta D}$. 
The learning rates $\etasgd^{(\ell+1)}$ should conspire to make the overall scale of the squared update $\frac{1}{L^2}N_\beta D$, i.e.
\[
    \frac{
        (\etasgd^{(\ell+1)})^2 N_\beta
    }{L^4D}
    \sim 
    \frac{N_\beta D}{L^2},
\]
so the learning rates must be
\begin{equation}
    \etasgd^{(\ell+1)}
    =
    \eta_0 DL
    \label{eq:sgd_lr_resid_layers}
\end{equation}
in the $\Wt{\ell+1}$ term. 
The same reasoning applies to the feature update term containing the MLP weights: the additional hidden width \(4D\) and the normalization \(1/\sqrt{4D}\) cancel, so the update scale again fixes \(\etasgd^{(\ell+1)}\sim DL\).
\\

\textbf{Last Layer.}
Updates to the last layer take the form
\begin{align*}
    \Exp{\big(\Delta z_\beta\big)^2}
    &=
    \frac{1}{(\sigma_{L+1}DN_\beta)^2}\bigg(
    \Exp{\big(\one_{N_\beta}^T\Delta\X{L}_\beta\W{L+1}{}\big)^2}
    +    
    \Exp{\big(\one_{N_\beta}^T\X{L}_\beta\Delta\W{L+1}{}\big)^2}
    \bigg)
    \\
    &=
    \frac{1}{D^2}\ 
    \Exp{\norm{\frac{1}{N_\beta}\one_{N_\beta}^T\Delta\X{L}_\beta}_2^2}
    + 
    \frac{(\etasgd^{(L+1)}\Delta_\alpha)^2}{\sigma_{L+1}^2 D^2}
    \Exp{\bigg(\frac{1}{N_\beta}\one_{N_\beta}^T \X{L}_\beta\nabla_{\W{L+1}{}}z_\alpha \bigg)^2}
    \\
    &\sim
    \frac{1}{D}
    + 
    \frac{(\etasgd^{(L+1)})^2}{\sigma_{L+1}^2 D^2}
    \Exp{\bigg(\frac{1}{N_\beta}\one_{N_\beta}^T 
    \X{L}_\beta\nabla_{\W{L+1}{}}z_\alpha\bigg)^2}
\end{align*}
where the gradient is 
\begin{align}
    \nabla_{\W{L+1}{}}z_\alpha 
    &=
    \nabla_{\W{L+1}{}}
    \bigg(
        \frac{1}{\sigma_{L+1}DN_\alpha}\one_{N_\alpha}^T\X{L}_\alpha\W{L+1}{}
    \bigg)
    \nonumber
    \\
    &= 
    \frac{1}{\sigma_{L+1}D}\bigg(\frac{1}{N_\alpha}\one_{N_\alpha}^T\X{L}_\alpha\bigg)^T
    \label{eq:grad_WLp1}
\end{align}
so that
\begin{align}
    \frac{(\etasgd^{(L+1)})^2}{\sigma_{L+1}^2D^2}
    \Exp{\bigg(\frac{1}{N_\beta}\one_{N_\beta}^T \X{L}_\beta\nabla_{\W{L+1}{}}z_\alpha\bigg)^2} 
    &=
    \frac{(\etasgd^{(L+1)})^2}{\sigma_{L+1}^4D^4}
    \Exp{
        \norm{
            \frac{1}{N_\beta}\one_{N_\beta}^T \X{L}_\beta
            \bigg(\frac{1}{N_\alpha}\one_{N_\alpha}^T \X{L}_\alpha\bigg)^T
        }_F^2
    }
    \nonumber
    \\
    &\sim \frac{(\etasgd^{(L+1)})^2}{\sigma_{L+1}^4 D^2},
    \label{eq:last_layer_grad_contribution}
\end{align}
where in the last line we assume that the quantity in the expectation brackets scales like $D^2$; this is motivated by the fact that the residual branches are scaled to keep  $\Exp{\norm{\X{L}_\alpha}_F^2}\sim N_\alpha D$ and $\Exp{\norm{\X{L}_\beta}_F^2}\sim N_\beta D$. 
Though the scale of $|z|$ is suppressed at initialization, we choose a learning rate $\etasgd^{(L+1)}$ so that the update $|\Delta z|$ is $\Theta(1)$: 
\begin{equation}
    \etasgd^{(L+1)} 
    \sim
    \eta_0 D \sigma_{L+1}^2.
    \label{eq:lr_last_layer}
\end{equation}
Thus, we may use $\sigma_{L+1}$ to modify the gradient scale in the last layer. 
\end{proof}
\vspace{1mm}


\begin{proposition}
\label{th:adam_lrs}
    \textbf{(Adam Learning Rates)}
    After one step of training on $\mathcal{G}_\alpha$ via Adam, the network features in the encoder, residual, and decoder layers receive $\Theta_{D,L}(1)$ updates on $\mathcal{G}_\beta$ if
    \begin{itemize}
        \item The learning rate in the first layer is 
        \begin{equation}
            \etaadam^{(0)} 
            =
            \eta_0\sigma_0
            \label{eq:lr_adam_first}
        \end{equation}
        \item The learning rates in the residual stream ($1\leq\ell\leq L$) are
        \begin{align}
            \etaadam^{(\ell)}
            = 
            \frac{\eta_0}{\sqrt{D}}.
            \label{eq:adam_lr_resid_layers}
        \end{align} 
        \item The learning rate in the last layer is
        \begin{equation}
            \etaadam^{(L+1)}
            =
            \eta_0\sigma_{L+1}
            \label{eq:adam_lr_last_layer}
        \end{equation}
    \end{itemize}
    Though different learning rates are required in each of the encoder, residual, and decoder layers, one may use a global effective learning rate $\etaadam=\eta_0/\sqrt{D}$ with the choices $\sigma_0=\sigma_{L+1}=1/\sqrt{D}$.  
\end{proposition}

\begin{proof}
Adam applies an RMSProp-style normalization to its gradients; we will exploit this fact to make a sign-GD approximation that
\begin{equation}
    \Delta \mathbf W_{ij}^{(\ell)}
    \approx
    -\etaadam^{(\ell)}
    \frac{
        \mathbf g_{ij}^{(\ell)}
    }{
        |\mathbf g_{ij}^{(\ell)}|
    }
    =
    -\etaadam^{(\ell)}
    \text{sign}(
        \mathbf g_{ij}^{(\ell)}
    )
    \label{eq:sign-gd-approx}
\end{equation}
where $\mathbf g_{ij}^{(\ell)} := \nabla_{\mathbf W^{(\ell)}_{ij}} \mathcal L$ is the loss gradient.
This approximation captures the Adam scaling when $\epsilon$ and weight decay are negligible (see Eqn.~\ref{eq:optimizers}; throughout we use it only to track width and depth dependence).

\paragraph{Encoder Layer}

From Eqn.~\ref{eq:first-layer-gradient}, the sign of the encoder weight gradients is dictated by  
\[
   \text{sign}(\mathbf{g}_{ij}^{(0)})
   = 
   \text{sign}(\Delta_\alpha)
   \text{sign}(\bar{\mathbf{X}}_{\alpha,i})
   \text{sign}(
        (\mathbf R^{(0)} \mathbf W^{(L+1)})_j
    )
\]
where $\bar{\mathbf X}_\alpha:=\frac{1}{N_\alpha}\boldsymbol 1^T\mathbf X_\alpha\in\R^{n_0}$ is the mean of the input node features. 
Then, the scale of each encoder feature update is
\[
    \Exp{
        \big(
            \Delta\mathbf X^{(0)}_{\beta,ui}
        \big)^2
    }
    =
    \frac{1}{\sigma_0^2 n_0}
    \Exp{
        \bigg(
            \sum_{j=1}^{n_0}
            \mathbf X_{\beta,uj}
            \Delta\mathbf W^{(0)}_{\alpha,ji}
        \bigg)^2    
    }
    \sim 
    \frac{(\etaadam^{(0)})^2}{\sigma_0^2 n_0}
    \bigg(
        \sum_{j=1}^{n_0}
        \mathbf X_{\beta,uj}
        \text{sign}(\bar{\mathbf X})_{\alpha,j}
    \bigg)^2
\]
where squaring removes the dependence of the weight update on the sign of the residual and the last-layer weights. 
Assuming that the quantity in parentheses scales like $\sqrt{n_0}$, and demanding this update be order-one in width and depth, we recover
\[
    \etaadam^{(0)} = \eta_0 \sigma_0.
\]

\paragraph{Residual Layers}
In the residual layers, the gradient directions are non-trivially correlated. 
As in Prop.~\ref{th:update-stability-gd-lr-graph-regression}, we need to consider several terms in the residual layers,
\[
    \Delta \X{\ell+1}_\beta
    = 
    \Delta \X{\ell}_\beta
     + 
     \frac{1}{L\sqrt{D}}
     \Delta\bigg(
        \X{\ell}_\beta\Wt{\ell+1}
     \bigg)
     +
     \frac{1}{2LD}
     \Delta\bigg(
        \X{\ell}_\beta\W{\ell+1}{1}\W{\ell+1}{2}
    \bigg)
\]
where the first scales like $\sqrt{N_\beta D}$ and the second and third like $\sqrt{N_\beta D}/L$.
In Prop.~\ref{th:update-stability-gd-lr-graph-regression}, we simplified the second term assuming GD (see Eqn.~\ref{eq:resid_grad_term}):
\begin{align*}
    \frac{1}{L^2D}
    \Exp{
    \norm{
        \Delta\bigg(
            \X{\ell}_\beta
            \Wt{\ell+1}
        \bigg)
    }_F^2}
    &= 
    \frac{1}{L^2D}
    \bigg(
        \Exp{
        \norm{
            \Delta \X{\ell}_\beta
            \Wt{\ell+1}
        }_F^2}
        +
        \Exp{
        \norm{
            \X{\ell}_\beta
            \Delta\Wt{\ell+1}
        }_F^2}
    \bigg)
    \\
    &\sim
    \frac{N_\beta D}{L^2} 
    + 
    \frac{(\etasgd^{(\ell+1)}\Delta_\alpha)^2}{L^2D}
    \Exp{
        \norm{
            \X{\ell}_\beta
            \nabla_{\Wt{\ell+1}} z_\alpha 
        }_F^2
    }
    \\
    &\sim
    \frac{N_\beta D}{L^2} 
    + 
    \frac{
        (\etasgd^{(\ell+1)})^2 N_\beta
    }{L^4D}
\end{align*}
To convert to an Adam learning rate, we must rescale the second term by a surrogate Adam gradient norm, which we take to be the RMS norm of the full gradient, $||\nabla_{\Wt{\ell}} z_\alpha||_F / D$. 
From Eqn.~\ref{eq:grad_Wl}, we can derive the norm of the gradient,  
\[
    \Exp{\norm{\nabla_{\Wt{\ell}}z_\alpha}_F^2}
    \approx
    \frac{1}{D^2L^2}
    \Exp{
    \norm{
        \frac{1}{N_\alpha}
        \one_{N_\alpha}^T
        \X{\ell}_\alpha
    }_F^2
    }
    \sim
    \frac{1}{DL^2}
\]
which, per-parameter, is $D^{-3}L^{-2}$. 
We therefore ask that
\[
    \frac{
        (\etaadam^{(\ell+1)})^2 N_\beta
    }{L^4D}
    \times
    D^3 L^2
    \sim
    \frac{N_\beta D}{L^2}
\]
so that the Adam learning rate in the residual layers should be 
\[
    \etaadam^{(\ell+1)} = \frac{\eta_0}{\sqrt{D}}
\]

\paragraph{Last Layer}
In proving Prop.~\ref{th:update-stability-gd-lr-graph-regression}, we found (see Eqn.~\ref{eq:last_layer_grad_contribution}) that for GD, the scale of updates to the output features are
\begin{align*}
    \Exp{\big(\Delta z_\beta\big)^2}
    &=
    \frac{1}{(\sigma_{L+1}DN_\beta)^2}\bigg(
    \Exp{\big(\one_{N_\beta}^T\Delta\X{L}_\beta\W{L+1}{}\big)^2}
    +    
    \Exp{\big(\one_{N_\beta}^T\X{L}_\beta\Delta\W{L+1}{}\big)^2}
    \bigg)
    \\
    &=
    \frac{1}{D^2}\ 
    \Exp{\norm{\frac{1}{N_\beta}\one_{N_\beta}^T\Delta\X{L}_\beta}_2^2}
    + 
    \frac{(\etasgd^{(L+1)}\Delta_\alpha)^2}{\sigma_{L+1}^2 D^2}
    \Exp{\bigg(\frac{1}{N_\beta}\one_{N_\beta}^T \X{L}_\beta\nabla_{\W{L+1}{}}z_\alpha \bigg)^2}
    \\
    &\sim
    \frac{1}{D}
    + 
    \frac{(\etasgd^{(L+1)})^2}{\sigma_{L+1}^4 D^2}.
\end{align*}
As before, to adapt to Adam, we need to rescale the second term by the Adam gradient norm surrogate $||\nabla_{\W{L+1}{}} z_\alpha||_F / \sqrt{D}$, the RMS norm of the full gradient. 
From Eqn.~\ref{eq:grad_WLp1}, 
\[
    \Exp{\norm{\nabla_{\W{L+1}{}} z_\alpha}_F^2}
    = 
    \frac{1}{\sigma_{L+1}^2D^2}
    \Exp{\norm{
        \frac{1}{N_\alpha}
        \one_{N_\alpha}^T\X{L}_\alpha
    }_F^2}
    \sim 
    \frac{1}{\sigma_{L+1}^2 D}
\]
where we recall that this scaling depends on the pooled-feature alignment assumption used in Eqn.~\ref{eq:last_layer_grad_contribution}. 
Per-parameter, then, the correction factor is $\sigma_{L+1}^{-2} D^{-2}$.
We then ask that
\[
    \frac{(\etaadam^{(L+1)})^2}{\sigma_{L+1}^4 D^2}
    \times 
    \sigma_{L+1}^{2} D^2
    \sim
    1
\]
so that the Adam learning rate in the last layer is 
\[
    \etaadam^{(L+1)} 
    = 
    \eta_0\sigma_{L+1}.
\]

\end{proof}

\subsection{Batching Effects}
\label{sec:batching-effects}

In the previous propositions, we considered only single graph updates. 
Here we comment on the effects of batching graphs, either in graph-level or vertex-level tasks. 
In the following remarks, we will fix two batches of $B$ graphs $\mathcal{B}_\alpha=\{\mathcal{G}_{\alpha,i}\}_{i=1}^B$ and $\mathcal{B}_\beta=\{\mathcal{G}_{\beta,i}\}_{i=1}^B$ with node features $\mathbf{X}_{\alpha,i}\in\R^{N_{\alpha,i}\times n_0}$ and $\mathbf{X}_{\beta,i}\in\R^{N_{\beta,i}\times n_0}$.
The numbers of nodes in each batch are $\Nba:=\sum_{i=1}^B N_{\alpha,i}$ and $\Nbb:=\sum_{i=1}^B N_{\beta,i}$ respectively.
We will stack the node features in each batch into matrices $\mathbf{X}_\alpha\in\R^{\Nba\times n_0}$ and $\mathbf{X}_\beta\in\R^{\Nbb\times n_0}$. 

\begin{remark}
\label{th:batching-effects_gd-vertex-classification}
    \textbf{(GD Node Regression)}
    Consider a minimal transfer GNN applied to a batch of graphs $\mathcal{B}_\alpha$ in a node-level regression task with one-dimensional targets 
    $\mathbf{Y}_\alpha\in\R^{\Nba}$. 
    After one step of gradient descent, the model's features and their updates remain $\Theta_{D,L}(1)$ if the learning rates are set as in Prop.~\ref{th:update-stability-gd-lr-graph-regression} with the following modification to the first-layer learning rate:
    \begin{equation}
        \etasgd^{(0)} 
        =
        \eta_0D\sigma_0^2
        \times
        \frac{n_0\sqrt{\Nbb}}{\Mab^*}
    \end{equation}
    with 
    \begin{equation}
        (\Mab^*)^2
        :=
        \Exp{
        \norm{
            \frac{1}{\Nba}\sum_{u=1}^{\Nba}
            \Delta_{\alpha,u}
            \mathbf{X}_\beta
            \mathbf{X}_{\alpha,u}^T
        }_2^2}
    \end{equation}
    where $\Delta_{\alpha,u}$ is the loss residual of node $u$ belonging to $\mathcal G_\alpha$. 
\end{remark}
\begin{proof}
In a node regression setting, the loss incurred is typically an average over all nodes in the batch, 
\[
    \mathcal{L}_\alpha = 
    \frac{1}{2\Nba}\sum_{u=1}^{\Nba} 
    (\mathbf Z_{\alpha,u} - \mathbf{Y}_{\alpha,u})^2.
\]
The gradient of this loss with respect to $\W{0}{}$ is
\[
    \nabla_{\W{0}{}}\mathcal L_\alpha
    = 
    -
    \frac{1}{\Nba}\sum_{u=1}^{\Nba}
    \Delta_{\alpha,u}
    \nabla_{\W{0}{}}\mathbf Z_{\alpha,u}
\]
where $\Delta_{\alpha,u}:=\mathbf{Y}_{\alpha,u}-\mathbf Z_{\alpha,u}$ is the loss residual. 
Updates to the first-layer activations have size 
\[
    \Exp{\norm{\Delta\X{0}_{\beta}}_F^2}
    = 
    \Exp{\norm{
        \frac{1}{\sqrt{n_0}\sigma_0}
        \mathbf{X}_\beta\Delta\W{0}{\alpha}
    }_F^2}
    =
    \frac{(\etasgd^{(0)})^2}{n_0\sigma_0^2}
    \Exp{\norm{
        \mathbf{X}_\beta
        \frac{1}{\Nba}\sum_{u=1}^{\Nba}
        \Delta_{\alpha,u}
        \nabla_{\W{0}{}} \mathbf Z_{\alpha,u}
    }_F^2}
\]
where the gradient takes an analogous form as the one derived in Eqn.~\ref{eq:first-layer-gradient},
\[
    \nabla_{\W{0}{}}\mathbf Z_{\alpha,u}
    =
    \frac{1}{\sigma_{L+1}D\sigma_0\sqrt{n_0}}
    \bigg(
        \mathbf{R}^{(0)}\W{L+1}{}\mathbf{X}_{\alpha,u}
    \bigg)^T
\]
where $\mathbf X_{\alpha,u}\in\R^{1\times n_0}$. 
With this, the update magnitudes are 
\begin{align*}
    \Exp{\norm{\Delta\X{0}_{\beta}}_F^2}
    &=
    \frac{(\etasgd^{(0)})^2}{n_0\sigma_0^2}
    \frac{1}{\sigma_{L+1}^2D^2\sigma_0^2n_0}
    \Exp{\norm{
        \mathbf{X}_\beta
        \frac{1}{\Nba}\sum_{u=1}^{\Nba}
        \Delta_{\alpha,u}
        \bigg(
            \mathbf{R}^{(0)}\W{L+1}{}\mathbf{X}_{\alpha,u}
        \bigg)^T
    }_F^2}
    \\
    &\approx
    \frac{(\etasgd^{(0)})^2}{n_0^2\sigma_0^4D}
    \Exp{\norm{
        \frac{1}{\Nba}\sum_{u=1}^{\Nba}
        \Delta_{\alpha,u}
        \mathbf{X}_\beta
        \mathbf{X}_{\alpha,u}^T
    }_2^2}
    \\
    &:= 
    \frac{(\etasgd^{(0)})^2}{n_0^2\sigma_0^4D}
    (\Mab^*)^2
\end{align*}
where we have defined
\[
    (\Mab^*)^2 := 
    \Exp{\norm{
        \frac{1}{\Nba}\sum_{u=1}^{\Nba}
        \Delta_{\alpha,u}
        \mathbf{X}_\beta
        \mathbf{X}_{\alpha,u}^T
    }_2^2}.
\]
Note that the presence of the loss residuals in the summation above is a key difference from $\Mab$ computed in the un-batched setting (see Eqn.~\ref{eq:Mab}). 
Demanding that the updates be order-one, we have
\[
    \frac{(\etasgd^{(0)})^2}{n_0^2\sigma_0^4D}
    (\Mab^*)^2
    \sim
    \Nbb D
    \ \ \implies \ \
    \etasgd^{(0)} = 
    \eta_0D\sigma_0^2 
    \times
    \frac{n_0\sqrt{\Nbb}}{\Mab^*}
\]

\end{proof}

\begin{remark}
\label{th:batching-effects_sgd-graph-regression}
    \textbf{(SGD Graph Regression)}
    Consider a minimal transfer GNN applied to a batch $\mathcal{B}_\alpha$ in a graph regression setting with scalar targets.
    After one step of gradient descent, the model's features and their updates remain $\Theta_{D,L}(1)$ if the learning rates are set as in Prop.~\ref{th:update-stability-gd-lr-graph-regression} with the following modification to the first-layer learning rate:
    \begin{equation}
        \etasgd^{(0)} 
        =
        \eta_0D\sigma_0^2
        \times
        \frac{n_0\sqrt{\Nbb}}{\Mab^B}
    \end{equation}
    with
    \begin{equation}
        (\Mab^B)^2 = 
        \Exp{\norm{
            \frac{1}{B}
            \sum_{i=1}^{B}
            \Delta_{\alpha,i}
            \frac{1}{N_{\alpha,i}}
            \one_{N_{\alpha,i}}^T
            \mathbf{X}_{\alpha,i}
            (\mathbf{X}_\beta)^T
        }_2^2}
    \label{eq:MabB}
    \end{equation} 
\end{remark}
\begin{proof}
For a 1D graph regression task, the batch-averaged MSE loss is 
\[
    \mathcal{L}_\alpha = 
    \frac{1}{2B}
    \sum_{i=1}^{B}
    \big(
        y_{\alpha,i} -
        z_{\alpha,i}
    \big)^2.
\] 
The gradient with respect to $\W{0}{}$ is 
\[
    \nabla_{\W{0}{}} \mathcal L_\alpha
    = 
    \frac{1}{B}\sum_{i=1}^{B}
    \Delta_{\alpha,i}
    \nabla_{\W{0}{}} z_{\alpha,i},
\]
and using the gradient from Eqn.~\ref{eq:first-layer-gradient},
\[
    \nabla_{\W{0}{}} z_{\alpha,i} 
    = 
    \frac{1}{\sigma_{L+1}\sigma_0 DN_{\alpha,i}\sqrt{n_0}}
    \bigg(
        \mathbf{R}^{(0)}
        \W{L+1}{}
        \one_{N_{\alpha,i}}^T
        \mathbf{X}_{\alpha,i}
    \bigg)^T,
\]
the scale of updates under SGD is 
\begin{align*}
    \Exp{\norm{\Delta\X{0}_\beta}_F^2} 
    &= 
    \frac{1}{n_0\sigma_0^2}
    \Exp{
        \norm{
            \mathbf{X}_\beta\Delta\W{0}{}
        }_F^2
    }
    = 
    \frac{(\etasgd^{(0)})^2}{n_0\sigma_0^2}
    \Exp{
        \norm{
            \mathbf{X}_\beta 
             \frac{1}{B}\sum_{i=1}^{B}
            \Delta_{\alpha,i}
            \nabla_{\W{0}{}} z_{\alpha,i}
        }_F^2
    }
    \\
    &\sim
    \frac{(\etasgd^{(0)})^2}{\sigma_0^4D n_0^2}
    \Exp{
    \norm{
         \frac{1}{B}\sum_{i=1}^{B}
         \Delta_{\alpha,i}
         \frac{1}{N_{\alpha,i}}
        \one_{N_{\alpha,i}}^T
        \mathbf{X}_{\alpha,i}
        \mathbf{X}_\beta^T
    }_F^2}
    \\
    &:=
    \frac{(\etasgd^{(0)})^2}{\sigma_0^4D n_0^2}
    (\Mab^B)^2
\end{align*}

For the scale of each update to be $\Theta(1)$, the learning rate in the first layer should conspire to make the overall scale of this squared Frobenius norm $\sim \Nbb D$, meaning
\[
    \etasgd^{(0)} = \eta_0D\sigma_0^2 \times
    \frac{n_0\sqrt{\Nbb}}{\Mab^B}.
\]
\end{proof}

\section{First-Layer Correction Factor as a Hyperparameter}
\label{sec:additional-fl-corr-study}

In this section we investigate the effect of treating $\Cab$ as a hyperparameter $C$ for the citation networks CiteSeer ($\widehat{C}_{\alpha\beta}\approx 22$) and PubMed ($\widehat{C}_{\alpha\beta}\approx 13.5$). 
We train models at scales $(L,D)\in\{(3,192),(6,512)\}$ using GD for 80~epochs on CiteSeer and 100~epochs on PubMed, sweeping a broad range of first-layer correction factors and base learning rates.
To investigate the speed at which models train, we measure the epoch at which models exceed a validation accuracy threshold, defined to be $85\%$ for PubMed and $70\%$ for CiteSeer.  
Fig.~\ref{fig:flc-epochs-heatmap} reports this epochs-to-threshold metric at each $(C, \eta_0)$ working point.
Any model that does not reach the target validation accuracy threshold is indicated by $X$. 
Fig.~\ref{fig:flc-epochs-heatmap} shows that under-correcting the first layer significantly slows training, and the range of acceptable $\eta_0$ values narrows for small $C$. 
Furthermore, near the predicted $\widehat{C}_{\alpha\beta}$, training is already substantially faster than the uncorrected ($C=1$) case. 
However, the fastest training models lie in a broad, high-$C$ basin where both widths reach the target in a
handful of epochs. 
We also note that as $C$ increases, the $\eta_0$ that minimizes epochs-to-threshold shifts left: large encoder learning rates need smaller global learning rates. 
Final val accuracy at the best $\eta_0$ for each~$C$ remains comparable across models with low epochs-to-threshold, indicating that $C$, at least for simple citation networks, primarily speeds convergence, rather than producing better generalization. 

\begin{figure}[t]
\begin{centering}
    \includegraphics[width=0.8\textwidth]{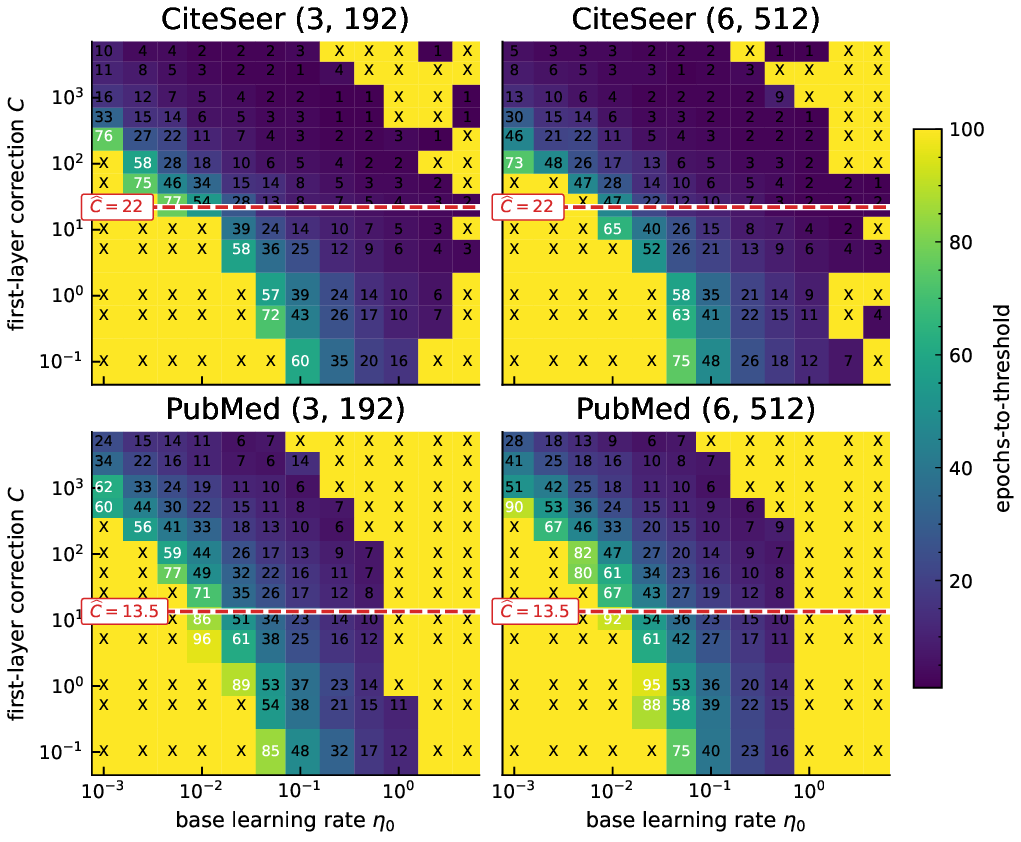}
    \caption{
    Epochs-to-threshold probe of first-layer learning rate correction factors as hyperparameters. $X$ indicates that the model did not reach the threshold in the allotted maximum epochs (80 for CiteSeer and 100 for PubMed). 
    }
    \label{fig:flc-epochs-heatmap}
\end{centering}
\end{figure}

\section{Additional Message Passing Normalization Study}
\label{sec:mnist-message-passing-normalization}
The MNIST Superpixels sweeps in Fig.~\ref{fig:normalization-sweep_mnist} show that in early training (10 epochs), only the normalized scenarios exhibit strong learning rate transfer around a narrow band of near-optimal $\eta_0$ values, and $\gamma$-normalization yields the best train and validation performance. 
Later in training (50 epochs), however, the range of near-optimal $\eta_0$ values broadens significantly for all three scenarios; the transfer remains sharp for the normalized scenarios and, notably, has improved for the unnormalized model. 

\begin{figure}[t]
\begin{centering}
    \includegraphics[width=0.8\textwidth]{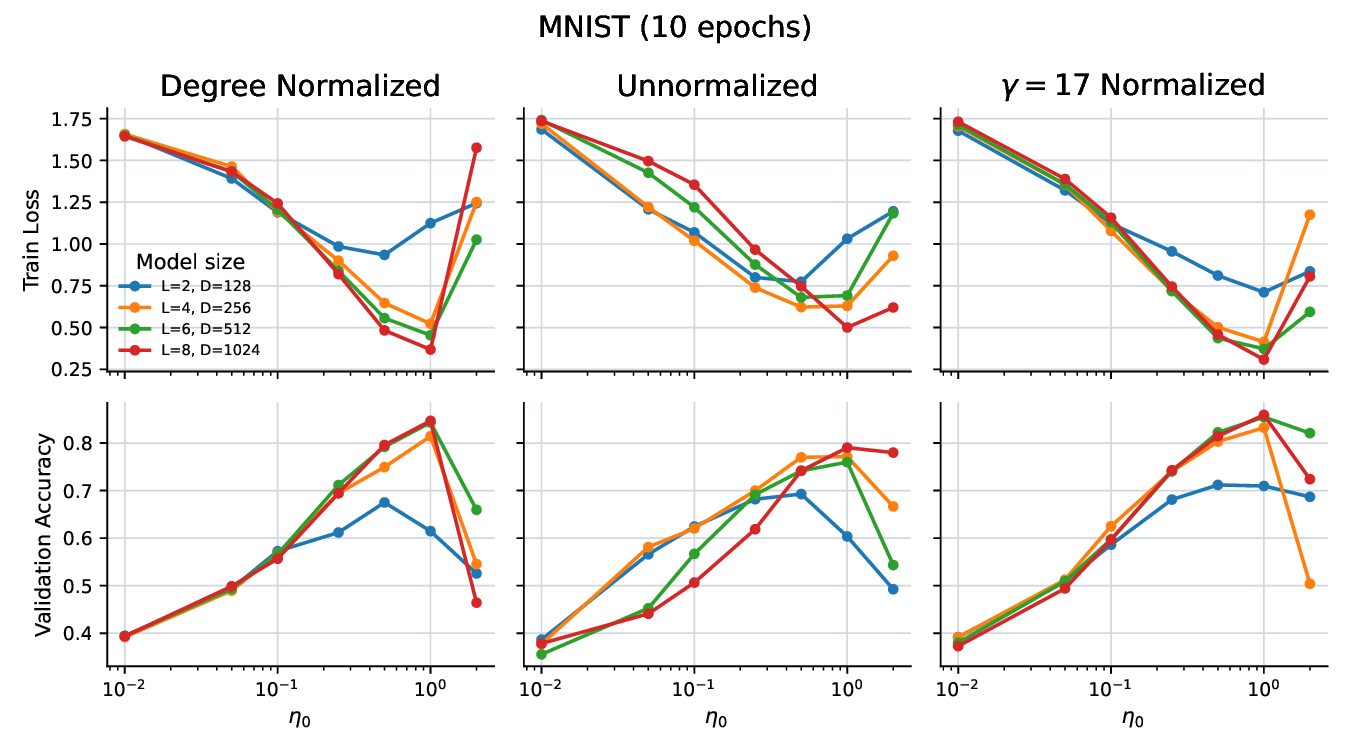}
    \includegraphics[width=0.8\textwidth]{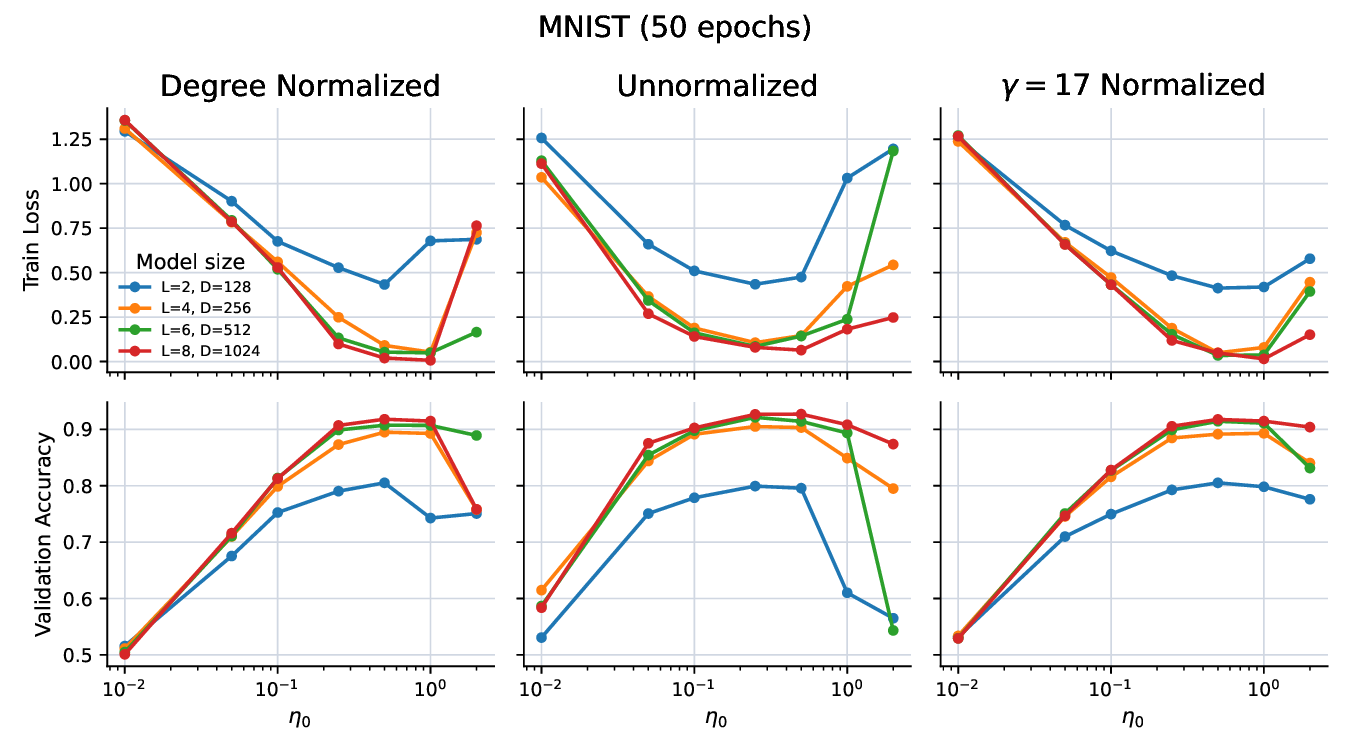}
    \caption{
    GNNs with degree normalization (left), no normalization (center), and $\gamma$ normalization (right) trained on the MNIST-SP dataset after 10 epochs (top) and 50 epochs (bottom). 
    }
    \label{fig:normalization-sweep_mnist}
\end{centering}
\end{figure}

\section{GenAI Usage Statement}
In this work, we used Generative AI to write code within the bounds of an experimental suite built by the authors.
The codebase has been carefully reviewed for consistency with what is reported in this paper. 
Furthermore, generative AI was used to verify the scaling arguments, compress notation, and edit this paper for both language and content. 
The authors reviewed and verified all AI-assisted code and text, and take full responsibility for their content.

\end{document}